\documentclass[10pt]{article}

\usepackage{amsmath,amsthm,verbatim,amssymb,amsfonts,amscd, graphicx, enumitem}
\usepackage{graphics}
\usepackage{centernot}
\usepackage{authblk}
\usepackage[T1]{fontenc}
\theoremstyle{plain}
\newtheorem{theorem}{Theorem}
\newtheorem{corollary}{Corollary}
\newtheorem{lemma}{Lemma}
\newtheorem*{remark}{Remark}
\newtheorem{proposition}{Proposition}

\newtheorem{definition}{Definition}

\usepackage[colorlinks,linkcolor=blue,citecolor=blue]{hyperref}

\usepackage[bottom]{footmisc}
\usepackage{caption}
\usepackage{subcaption}
\usepackage{helvet}  
\usepackage{courier}  
\usepackage{url}  
\usepackage{graphicx}  
\usepackage{multirow}
\usepackage{amsthm}
\usepackage{color}
\usepackage{MnSymbol}
\usepackage{makecell}
\usepackage{arydshln}
\usepackage{amsmath}
\usepackage[dvipsnames]{xcolor}
\usepackage{caption} 
\usepackage{natbib}
\usepackage{bbm}

\usepackage{textcomp}
\usepackage{wrapfig}
\usepackage{algorithm}
\usepackage{algorithmic}

\usepackage{csquotes}

\newcommand{\E}{\mathbb{E}}

\newcommand{\Tr}{{\rm Tr}}

\title{\textbf{Does SGD Seek Flatness or Sharpness? \\ 
  An Exactly Solvable Model}}
\author{Yizhou Xu$^1$, Pierfrancesco Beneventano$^2$, Isaac Chuang$^2$, Liu Ziyin$^{2,3}$\\
$^1$\textit{École Polytechnique Fédérale de Lausanne}\\
$^2$\textit{Massachusetts Institute of Technology}\\
$^3$\textit{NTT Research}}
\begin{document}

\maketitle

\begin{abstract}
A large body of theory and empirical work hypothesizes a connection between the flatness of a neural network's loss landscape during training and its performance. However, there have been conceptually opposite pieces of evidence regarding when SGD prefers flatter or sharper solutions during training. In this work, we partially but causally clarify the flatness-seeking behavior of SGD by identifying and exactly solving an analytically solvable model that exhibits both flattening and sharpening behavior during training. In this model, the SGD training has no \textit{a priori} preference for flatness, but only a preference for minimal gradient fluctuations. This leads to the insight that, at least within this model, it is data distribution that uniquely determines the sharpness at convergence, and that a flat minimum is preferred if and only if the noise in the labels is isotropic across all output dimensions. When the noise in the labels is anisotropic, the model instead prefers sharpness and can converge to an arbitrarily sharp solution, depending on the imbalance in the noise in the labels spectrum. We reproduce this key insight in controlled settings with different model architectures such as MLP, RNN, and transformers.

\end{abstract}

\section{Introduction}

The flatness and sharpness of the loss landscape have become a major object of study for understanding the optimization and generalization of neural networks \citep{hochreiter1997flat, keskar2016large}, not to mention various algorithms based on these two mechanisms \citep{foret2020sharpness}. For example, flatness is one of the most common hypotheses for why neural networks generalize: even when many parameter settings fit the training data perfectly, solutions that lie in flatter regions of the loss landscape are often argued to be more stable and to generalize better. This (rather questionable) intuition has motivated a large literature connecting SGD’s step size and stochasticity to an implicit preference for flatter minima, typically operationalized through Hessian-based sharpness measures \cite{keskar2016large,wu2018sgd,xie2020diffusion,wu2023implicit}. At the same time, an apparently contradictory phenomenon has become increasingly prominent in modern training: sharpness often increases over the course of optimization, and eventually reaches the edge of stability (EoS), where the curvature is large, and training is only marginally stable \cite{cohen2021gradient,andreyev2024edge}. This coexistence of the two tendencies ``SGD seeks flatness” and ``SGD seeks sharpness” is what we will call the sharpness ``paradox" -- of course, this is not a paradox in the literal sense, but a reflection of our lack of understanding about whether SGD has an intrinsic tendency towards learning flatter or sharper solutions. Resolving this ``paradox" is quite difficult, especially because of the existence of many confounding variables: sharpness trajectories are confounded by initialization (a flatter start will lead to sharpening and vice versa) and by run-to-run variability.

In this work, we give an exact solution of the sharpness to a mathematical model under the so-called ``minimal-fluctuation" constraint, which allows us to strip away these confounders and identify the precise cause of the flattening and sharpening within this model. In this model, the global minima form a highly degenerate manifold containing solutions with a wide range of sharpness values  -- including arbitrarily sharp ones  -- yet the learning dynamics of SGD selects a solution whose sharpness is unique. This lets us answer the implicit-bias question in its cleanest form: when many equally good minima exist, which curvature does SGD pick, and what factors determine that choice? Our main contributions are:
\begin{enumerate}[noitemsep,topsep=0pt, parsep=0pt,partopsep=0pt, leftmargin=13pt]
    \item In the regime where SGD minimizes fluctuations, we identify an exactly solvable model of sharpness / flatness;
    \item We precisely identify the causes of flatness-seeking behavior of SGD in this model:
    \begin{enumerate}[noitemsep,topsep=0pt, parsep=0pt,partopsep=0pt]
        \item When the label uncertainty is isotropic, the training always prefers the flattest of all solutions;
        \item Otherwise, the label uncertainty spectrum $\Sigma_\epsilon$ determines the converged sharpness, which is inversely proportional to the condition number of $\Sigma_\epsilon$.
    \end{enumerate}
\end{enumerate}
These predictions are tested in nonlinear models and realistic data distributions, showing that the messages carry beyond the analytical model we solved. 
The starting point of our theory is the recent theoretical advancement that shows that SGD training has a intrinsic preference for solutions that have a small gradient fluctuation \cite{roberts_sgd_2021,smith2021origin, beneventano2023trajectories,ziyin2025neural}, and our empirical success of in explaining real neural networks lends further support to the minimal-fluctuation theory. In the minimal-fluctuation theory, the crucial point -- and the source of the paradox -- is that fluctuation and sharpness are not the same object. They align only under special conditions, so minimizing fluctuation does not generally imply minimizing sharpness. The exact model we solve is a deep linear network trained on a linear teacher with noisy labels. In this model, the SGD-selected sharpness can be written explicitly in terms of the input covariance, the teacher map, the depth, and the label-noise\footnote{To clarify, throughout our work, we use the word ``label noise" in the broad sense. It could mean a dynamically injected noise or could be an inherent static noise in the labels such that the model cannot perfectly interpolate the data.} covariance across output dimensions. 

The key consequence of the theory is to reframe the sharpness debate in data-geometric terms: whether SGD appears flatness-seeking or sharpening with respect to GD is situation-dependent, and the controlling knob is the anisotropy of effective noise on the labels. 

\begin{wrapfigure}{r}{0.4\linewidth}
    \centering
    \includegraphics[width=\linewidth]{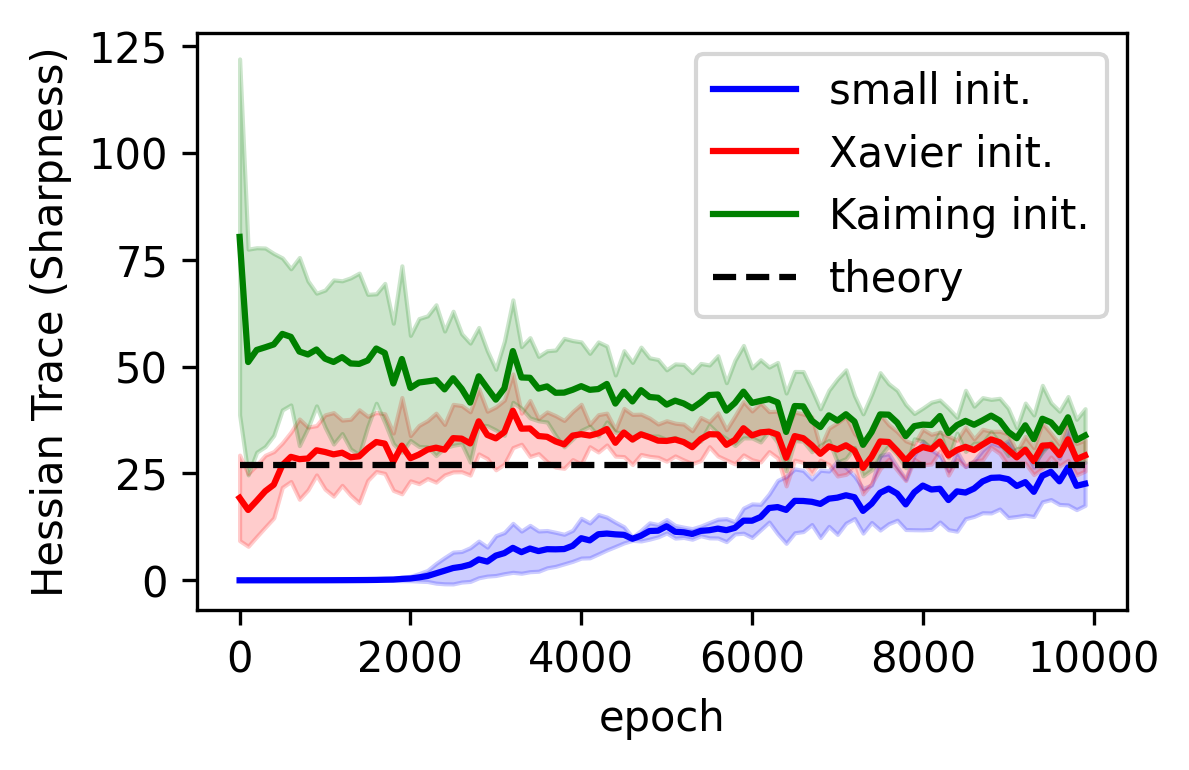}
    \caption{\small We reproduce the experiment from \cite{ziyin2024parameter}, where a deep matrix factorization problem trained with SGD converges to the same sharpness, even if the initial sharpness depends strongly on the initialization scale. The theoretical line is the prediction of Theorem \ref{theo:main}.}
    \label{fig:initialization}
\end{wrapfigure}

\section{Sharp or flat?}
\label{sec:sharp_or_flat}
\paragraph{Sharpness.} A long-standing consensus in deep learning theory suggests that the implicit bias of Stochastic Gradient Descent (SGD) drives the model toward flat minima. This line of thought \cite{hochreiter1997flat,keskar2016large,foret2020sharpness} argues that flat regions are more robust to parameter perturbations, which translates to superior generalization. A vast body of theoretical work seeks to explain such an implicit bias from the perspectives of gradient noise and dynamical stability \cite{zhu2018anisotropic,wu2018sgd,wu2022alignment,wu2023implicit,mulayoff_exact_2024}. However, recent research into the discrete dynamics of neural network training reveals that SGD does not always settle into the flattest possible regions. Instead, it often converges to regions that are significantly ``sharp", a phenomenon known as the progressive sharpening \cite{xing_walk_2018,jastrzebski_relation_2019,jastrzebski_break-even_2020,cohen2021gradient,andreyev2024edge}. 
Thus, the preference of SGD regarding the sharpness remains a paradox.

\paragraph{Arbitrariness of initialization.} However, due to the large number of confounding variables, it is difficult to say and understand when a flatter minimum or a sharper minimum is preferred. The first confounding variable is the arbitrariness of the solution. It is common for us to measure the evolution of sharpness over training steps and argue that the model prefers a sharper or flatter solution by looking at how sharpness changes during training. However, this perspective is biased by the arbitrariness of the initialization scheme we use. Imagine a landscape where the global minimum has a unique, nonzero level of sharpness, but initializations can have arbitrary levels of sharpness. Then, conditioning on the training trajectories that take the model to the global minimum, a flatter initialization implies that the training must, on average, sharpen the landscape, while a sharper initialization must, on average, flatten. In this case, it is difficult to tell whether the learning dynamics favor flatness or sharpness, as this largely depends on the initialization's arbitrariness. The second confounding variable is the potential randomness of the outcome. It is also fully possible that every initialization (or sampling process) leads to a different local minimum with different values of sharpness, and it becomes difficult to compare across settings. Thus, ideally, the ideal theoretical model should remove this arbitrariness. Identifying a theoretical model where the models only converge to unique levels of sharpness is another great step towards understanding the flatness-seeking behavior. See Figure~\ref{fig:initialization} for a demonstration of this effect. 



\paragraph{Intrinsic sharpness.} Therefore, one can imagine the following scenario being the best for understanding the flatness-seeking behavior of learning algorithms: (a) construct a loss landscape with a degenerate global minima manifold, containing solutions with different levels of sharpness; (b) condition on the learning trajectories reaching a global minimum, and study how the learning dynamics determine the average level of sharpness. In this scenario, there are clearly two contributions to the sharpness of the learned solution: (a) the intrinsic sharpness required to learn the problem: for any given task and any given architecture, the minimal level of sharpness required to learn it should be nonzero; (b) the learning-algorithm-biased sharpness: when there is a manifold of solutions with different levels of sharpness values, which one among these is chosen by SGD. Conceptually, we can write (using $T$ to denote some metric of sharpness):
\begin{equation}
\label{eq:curvature_split}
    T = T_{\rm intrinsic} + \Delta_{\rm learning}.
\end{equation}
The first term is present in any learning algorithm that can reach a global minimum, and so whatever effect that only changes the first term cannot be argued to be a ``preference" of the learning algorithm. The second term is purely determined by the preference learning algorithm and the choice of its hyperparameters. What affects the second term can thus be called the ``implicit flatness bias." Of course, it is equally important to study what determines the first term, which we will also answer in our model.



\paragraph{Novel SGD-specific sharpening.}
Prior work on {progressive sharpening} and the {edge-of(-stochastic)-stability} documents that curvature can increase during training, but the underlying mechanism is fundamentally \emph{landscape-driven}: analogous sharpening is observed for GD and Gradient Flow (GF) as well, and interventional studies attribute the effect to landscape slopes pointing towards sharper areas, rather than to the SGD-specific stochasticity \citep{jastrzebski_break-even_2020,cohen2021gradient,cohen2022adaptive,andreyev2024edge}. 
Our mechanism is qualitatively different. 
By contrast, SGD \emph{does} induce a selection rule: it approximately minimizes the entropic objective~\eqref{eq: free energy}, and therefore selects minimizers of the gradient-fluctuation functional $S(\theta)$ over the minimizer manifold (Lemma~\ref{lemma: minimal fluctuation}). 
{GF/GD are insensitive to the label-noise covariance} $\Sigma_\epsilon$ (they depend on label {first moments}, not on second moments), while {SGD is not}.
In our solvable model, the minibatch-induced term selects among a {degenerate manifold of global minima} by minimizing gradient fluctuations, and this selection depends explicitly on the {anisotropy} of $\Sigma_\epsilon$.
As a consequence, SGD can move {along the minimizer manifold} toward solutions that are {strictly sharper}---and can be arbitrarily sharp---even though much flatter global minima coexist.
This is not ``sharpening on the way to convergence'', but a {noise-conditioned implicit bias at convergence}, see also Fig. \ref{fig:GD}.

In our model, we will find that $\Delta_{\rm learning}$ is the SGD-specific contribution arising from fluctuation minimization (vanishing in the isotropic-noise case and scaling with $\kappa(\Sigma_\epsilon)$ in our model; Cor.~\ref{cor1}). 
Consistent with this interpretation, a direct intervention that switches from SGD to full-batch GD arrests the sharpening drift (Fig.~\ref{fig:GD}). 
Overall, this \textit{challenges, if not overturns, the common heuristic} ``SGD seeks flatness'' narrative: in our model, SGD primarily seeks {minimal fluctuations}, and flatness emerges only in the contrived regimes where fluctuation and sharpness are aligned.



\section{An exactly solvable model}
In this section, we first motivate and formally define the problem setting, and we then present the main theorem, whose proof is given in the appendix. Readers wishing to learn more about the intuition may start by reading the general theoretical framework outlined in Section~\ref{sec: general theory}. We will discuss the various implications of the theory in the next main section. 

\paragraph{Notation:} We use $\theta$ to denote the set of all trainable parameters. When used without a subscript, $\nabla$ denotes the gradient with respect to $\theta$. $\lambda_{\max}(\Sigma)$ denotes the maximal eigenvalue of $\Sigma$. $\nabla^2 z$ denotes the Hessian matrix of a scalar function $z$. Throughout this work, we use $\E$ to denote the average over the training set $\mathcal{D}_{\rm train}$, which is not necessarily a finite set. $I_d$ is the identity operator for $\mathbb{R}^d$, and $d$ is often omitted when the context is clear. Other notations are introduced in the context. 

\subsection{Model and loss function}

We study a $D$-layer deep linear network. Specifically, the model is defined as
\begin{equation}
    f_\theta(x) = W_D\cdots W_1x.
\end{equation}
where $\theta:=\{W_k\}_{k=1}^D$ denotes the collection of all the weight matrices, and $x$ is the input data. 
The model is trained with the MSE loss on a regression task, where the data-wise stochastic loss is
\begin{equation}
\ell(x,y,\theta)=\frac{1}{2}||y-W_D\cdots W_1x||^2,
\end{equation}
where $y$ is the label associated with $x$. 
In practice, we train this model with SGD, but the training dynamics of SGD on this model are difficult to solve analytically, and we thus adopt an alternative approach, which can be seen as answering the question of ``if the training proceeds with SGD, then what kind of objective function is the algorithm actually minimizing?" This alternative problem, in contrast, can be exactly solved to its full generality. This alternative objective function will be discussed in Section~\ref{sec: minimal fluctuation}.

The labels are generated by a linear teacher $V$ with noise $\epsilon$
\begin{equation}
y=Vx+\epsilon\in\mathbb{R}^{d_y},    
\end{equation}
where $\mathbb{E}[x]=0$ and $\mathbb{E}[\epsilon]= 0$. 

We define the covariances $\Sigma_x := \mathbb{E}_x[xx^T]$ and $\Sigma_\epsilon := \mathbb{E}_\epsilon[\epsilon\epsilon^T]$. Note that all input, label, $V$, and noise are allowed to have arbitrary dimensions as long as their dimensions are mutually consistent. Under this data generation model, the global minima of the empirical loss $L$ satisfy:
\begin{equation}
    W_D...W_1 x = V x
\end{equation}
for all $x$ in the training set.

As a key part of the analysis, note that the model has many symmetries (namely, loss-invariant directions), which we call the matrix rescaling symmetry. For an arbitrary invertible matrix $A$, and an arbitrary layer index $i$, we have
\begin{align}
    f(W_i, W_{i-1}) &:= W_D\cdots W_iW_{i-1}\cdots W_1x\\
    &= f(W_iA,A^{-1} W_{i-1}).
\end{align}
This also implies that there are infinitely many global minima for the deep linear network. 
As we will show in the next section, this implies that a global minimum for $L$ exists with arbitrarily high sharpness.

\subsection{Sharpness}
For most of the manuscript, we use the following metric of sharpness:
\begin{definition}
The sharpness $T(\theta):=\Tr\E[\nabla^2\ell(x,y,\theta)]$.
\end{definition}
$T$ has been used in many prior works as a metric of sharpness and is often found to better represent the stability of the learning dynamics \cite{li2021happens,wu2023implicit,ziyin2024parameter}, which is one of the primary reasons for studying the sharpness of the landscape. It also offers an upper and lower bound to the largest eigenvalue $\lambda_{\max}$ of the Hessian:
\begin{equation}
    T / d \leq  \lambda_{\max} \leq T.
\end{equation}
Thus, $T(\theta)$ serves as a reliable proxy for the stability of the solution. 
In the asymptotic limit where $T \to \infty$, $ \lambda_{\max}$ must also diverge.

The necessity of studying SGD's implicit bias is highlighted by the fact that the loss landscape contains paths to arbitrarily sharp minima. The following lemma shows that for the class of symmetries called exponential symmetries \cite{ziyin2024parameter}, we can always find a global minimum with infinite sharpness:
\begin{lemma}
\label{lemma:sharpness}
Assume that $A$ is a symmetric matrix and that for any $x,y,\theta$, $\ell(x,y,e^{\lambda A}\theta)=\ell(x,y,\theta)$. Moreover, assume that $A\mathbb{E}_{x,y}\nabla^2\ell(x,y,\theta)\neq0$. Then, $\limsup_{|\lambda|\to+\infty}|T(e^{\lambda A}\theta)|=+\infty$.
\end{lemma}
The matrix rescaling symmetry is a type of exponential symmetry as well and obeys Lemma~\ref{lemma:sharpness}. That symmetry implies the existence of sharp solutions is first identified and proved in \cite{Dinh_SharpMinima} for the case of simple rescaling symmetries, and lemma~\ref{lemma:sharpness} can be seen as the most general form of this result.

As an example of its meaning, consider its application to the deep linear network with $D=2$. For $\ell(x,y,\theta)=\frac{1}{2}||y-W_2W_1x||^2$, we have
\begin{equation}
T(W_2,W_1)=||W_2||_F^2\Tr\Sigma_x+d_y\text{Tr}(W_1^T W_1 \Sigma_x),
\end{equation}
In this case we have $\ell(x,y,W_2,W_1)=\ell(x,y,\lambda W_2,\lambda^{-1}W_1)$ but
\begin{equation}
\lim_{|\lambda|\to+\infty}T(\lambda W_2,\lambda^{-1}W_1)=+\infty
\end{equation}
for any $W_1,W_2$.

As the main theorem~\ref{theo:main} below will prove, at the global minima, the deep linear network will (1) have a unique minimal sharpness value $T_{\min}$ that only depends on the problem and will (2) converge to a unique sharpness value $T\geq T_{\min}$, this means that the actually learned level of sharpness can be decomposed into two contributions:
\begin{equation}
    T = T_{\min} + \Delta_{\rm SGD},
\end{equation}
where $\Delta_{\rm SGD}\geq 0$, $T_{\min}$ is the minimal level of sharpness that is required in order to fit the data distribution (given the architecture) and needs to be obeyed by any learning algorithm that is capable of reaching the global minimum, and $\Delta_{\rm SGD}$ is the level of sharpness due to the implicit bias of SGD (namely, due to the fluctuation minimization effect that we will introduce in the next section).

The more standard quantity $\lambda_{\max}$ is also of interest because it directly relates to the recent popular problem of edge of stability (EoS). In general, while our theory allows one to compute $\lambda_{\max}$ directly, its functional form is so complicated that we do not present it explicitly\footnote{Note that this is not a problem of our theory, but a problem of generally computing the largest eigenvalue of any matrix.}. That being said, we will discuss in Section~\ref{sec: insights} a few cases where $\lambda_{\max}$ can indeed be written in easily interpretable form.

\begin{figure*}[t!]
    \centering
    \includegraphics[width=0.3\linewidth]{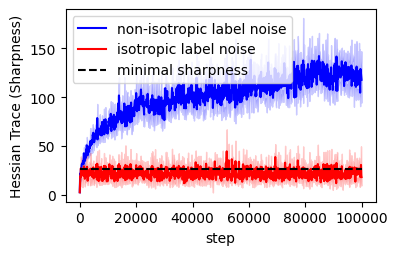}
    \includegraphics[width=0.3\linewidth]{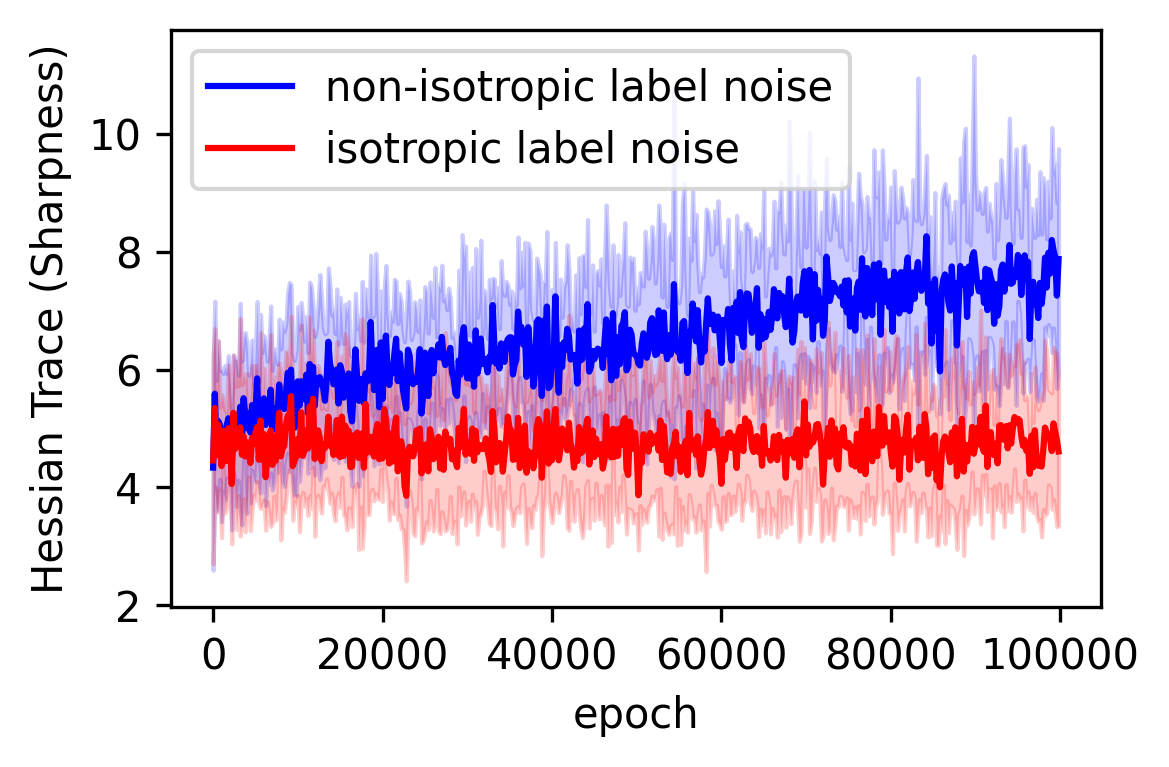}
    \includegraphics[width=0.3\linewidth]{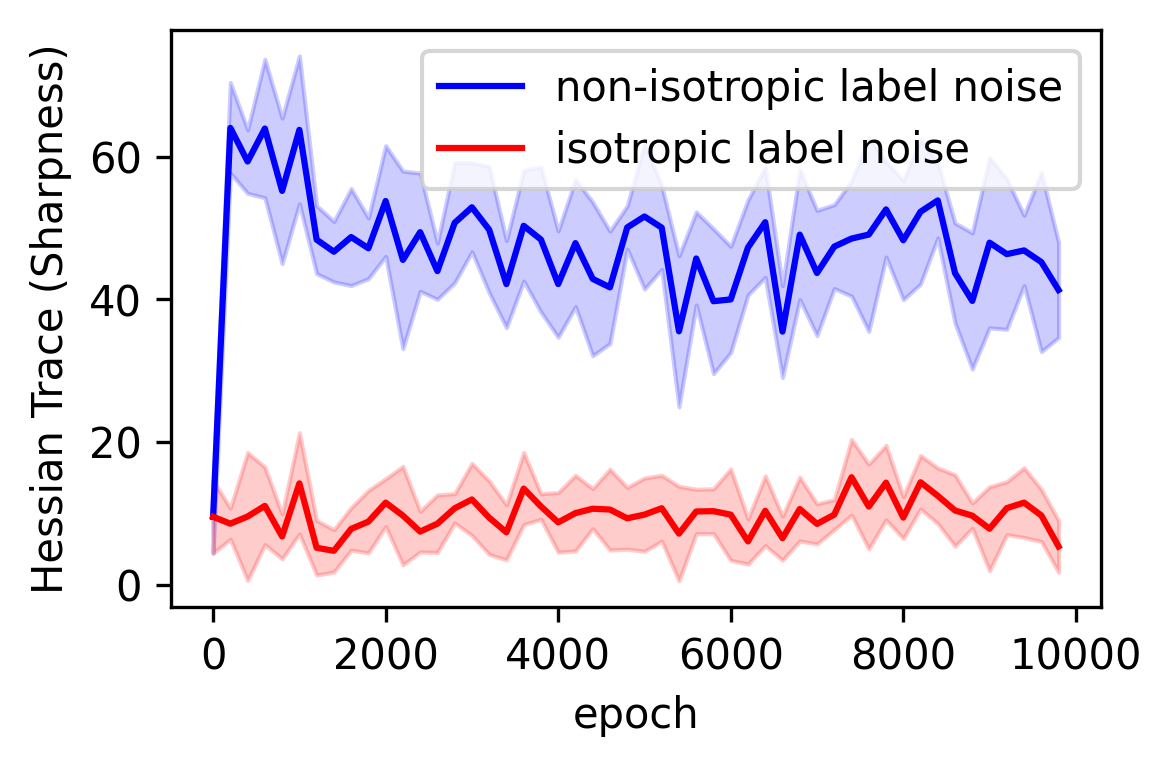}
    \caption{\small Non-isotropic noise in the labels leads to progressive sharpening. \textbf{Left}: linear networks, with the minimal sharpness predicted by \eqref{eq:minT}. \textbf{Middle}: Two layer ReLU networks trained under a teacher-student setting. \textbf{Right}: Two layer ReLU networks trained on MNIST. Each line is averaged over five trials shown with the standard error. See Appendix \ref{app:exp} for experiment details and more experiments.}
    \label{fig:label_noise}
\end{figure*}
\begin{figure*}[t!]
    \centering
    \includegraphics[width=0.27\linewidth]{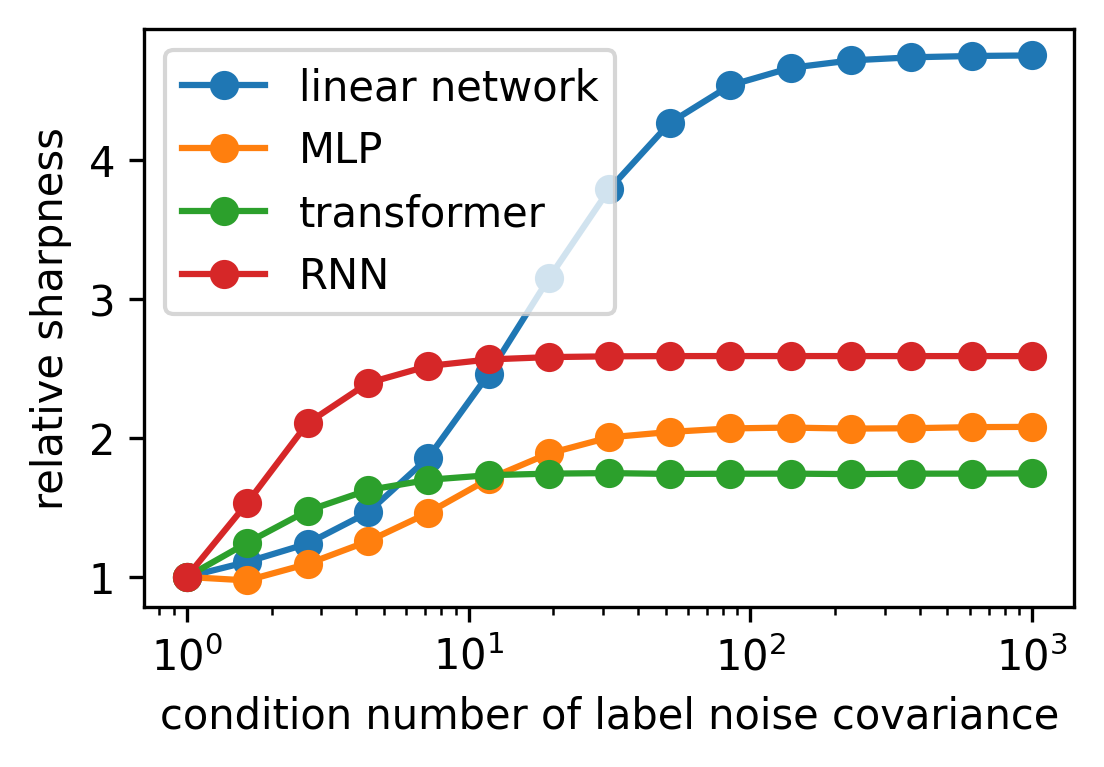}
    \includegraphics[width=0.28\linewidth]{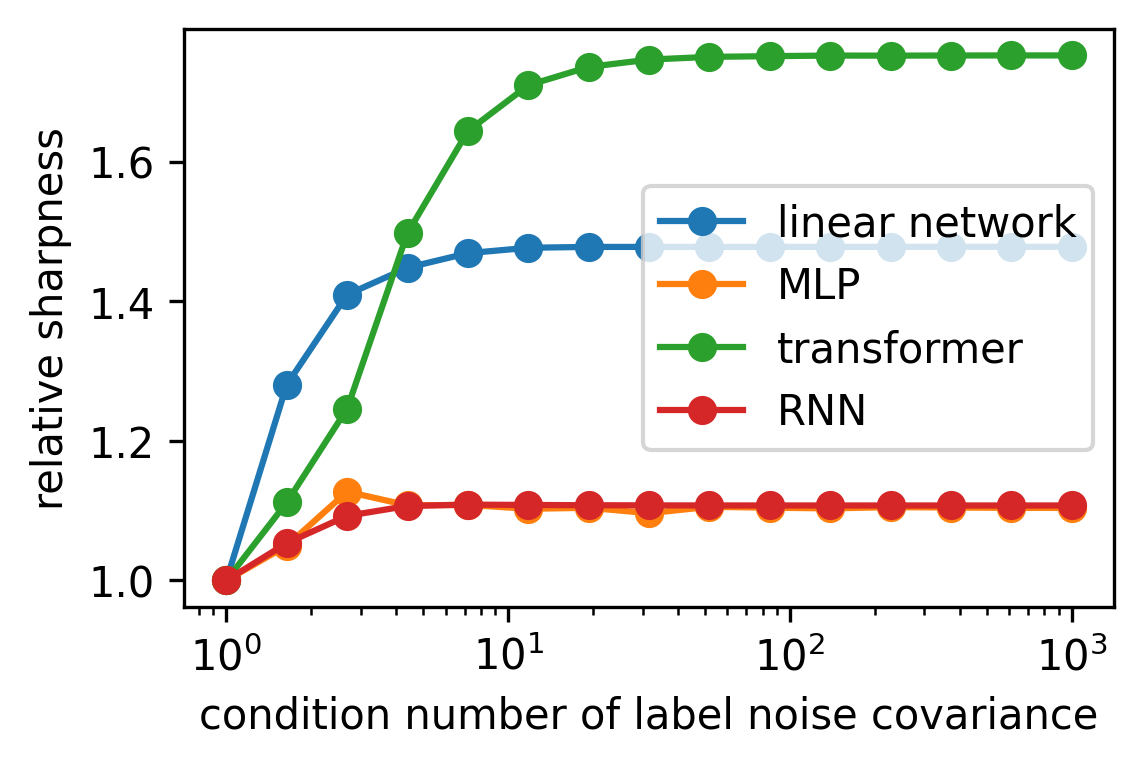}
    \caption{\small Across all four models, and for both regression and classification tasks, we observe that the learning process yields a sharper solution as the noise in the labels becomes increasingly non-isotropic. We define relative sharpness as the ratio of the Hessian trace under non-isotropic noise to that under isotropic noise. \textbf{Left}: MSE loss. \textbf{Right}: Cross entropy loss.}
    \label{fig:sharpness_condition_number}
\end{figure*}

\subsection{Minimal Fluctuation Ansatz}\label{sec: minimal fluctuation}

A line of works has shown that training with SGD at a finite learning rate and minibatch sampling is equivalent to solving the following perturbed optimization problem, which we refer to as the entropic loss \cite{roberts_sgd_2021,smith2021origin,beneventano2023trajectories,ziyin2025neural}:
\begin{align}
\label{eq: free energy}
F_{\eta}(\theta)&=\underbrace{\mathbb{E}_{x,y}\ell(x,y,\theta)}_{\text{learning, symmetry}} +\underbrace{\frac{\eta}{4}\mathbb{E}_\mathcal{B}||\mathbb{E}_{(x,y)\in\mathcal{B}}  \nabla \ell(x,y,\theta)||^2 }_{\text{due to discretization error, noise, $:=S(\theta)$}},
\end{align}
where $\mathcal{B}$ denotes a mini-batch and $\eta$ is the learning rate. 
Using this modified entropic loss allows us to characterize the implicit bias of SGD via the following lemma.
\begin{lemma}[Minimal Fluctuation Lemma, informal]\label{lemma: minimal fluctuation}
In the limit $\eta \to 0^+$, the global minima of $F_{\eta}(\theta)$ are 
\begin{equation}
    \arg \min_{\theta \in \mathcal{M}} S(\theta),
\end{equation}
where $M = \{ \theta \mid \mathbb{E}_{x,y} \ell(x,y, \theta) = \inf_{u} \mathbb{E}_{x,y} \ell(x,y, u) \}$ is the the global minima manifold.

\end{lemma}

Since modern neural networks are typically over-parameterized, the loss landscape $\mathbb{E}_{x,y} \ell(x,y, \theta)$ contains a manifold of global minima. The lemma suggests that SGD does not treat these minima equally but implicitly prefers those with minimal fluctuation $S$. Notably, this bias vanishes for full-batch gradient descent (GD), as the fluctuation term disappears at any global optimum.

A key mechanism of the proof is that different symmetry directions along the matrix rescaling symmetry actually have different levels of fluctuation (often also corresponding to different levels of sharpness, as shown by Lemma~\ref{lemma:sharpness} ) -- thus, the preference of the learning algorithm for the minimal-fluctuation solution causes the learning algorithm to also moves towards a certain level of sharpness.


\subsection{Main Result}\label{sec: main result}

Let $V' := \sqrt{\Sigma_\epsilon} V \sqrt{\Sigma_x}$. Let its SVD decomposition be given by $V' = L S R$. We assume the network width $d \geq \text{rank}(V')$\footnote{$d$ is defined as the the smallest dimension of the output and input dimensions of $W_i$ for all $i \in [D]$. When $d<\text{rank}(V')$, it learns the largest $d$ eigenvalues, and the result is similar.}. Under the SGD limit $\eta \to 0^+$, we obtain the following closed-form result for the sharpness:
\begin{theorem}
For all global minimum of \eqref{eq: free energy}, the spectrum of the Hessian is identical (determined by $\Sigma_x,\Sigma_\epsilon$ and $V$), and the sharpness is given by
{\small
\begin{equation}
\begin{aligned}
&T(\theta)=(D-1)\left(\Tr S\right)^{(D-2)/D}(\Tr\Sigma_\epsilon\Tr\Sigma_x)^{1/D}\Tr[\Sigma_\epsilon^{-1}LSL^T]+d_y\left(\Tr S\right)^{2(D-1)/D}\frac{(\Tr\Sigma_x)^{1/D}}{(\Tr\Sigma_\epsilon)^{(D-1)/D}}.
\end{aligned}
\label{eq:sharpness}
\end{equation}}
In addition, the minimal sharpness of the global minimum of $\mathbb{E}_x\ell(x,\theta)$ is given by
\begin{equation}
T_{\min} := \min_\theta T(\theta)=D(\Tr\hat{S})^{2(D-1)/D}(d_y\Tr\Sigma_x)^{1/D},
\label{eq:minT}
\end{equation}
where $\hat{S}$ is the singular values of $V\sqrt{\Sigma_x}$.
\label{theo:main}
\end{theorem}

Theorem \ref{theo:main} reveals a striking property: although there exist infinitely many global minima with varying (and even infinite) sharpness, SGD consistently selects a solution with a \textbf{unique} sharpness level. This suggests that the final sharpness is independent of initialization, provided the model converges to a global optimum. Moreover, the first term of \eqref{eq:sharpness} suggests that $D-1$ layers contribute to the sharpness in the same way, while the last layer (the second term of \eqref{eq:sharpness}) contributes differently. At the minimal sharpness solution \eqref{eq:minT}, however, all layers contribute in the same way. We will analyze the implications of this theorem in detail in Section~\ref{sec: insights}.

\section{Insights}\label{sec: insights}
A great merit of exactly solvable models is that every effect and phenomenon within the model can be precisely and causally understood. This section is thus devoted to naming and understanding the key observations and insights of Theorem \ref{theo:main}. We then present experiments on a diverse set of model architectures to demonstrate the qualitative relevance of the prediction. The proofs of the theoretical results are also presented in the appendix.

\paragraph{Noise Imbalance as a Source of Sharpness.} 

\begin{corollary}
Suppose $\Sigma_\epsilon$ and $V\sqrt{\Sigma_x}$ have full rank. Let the condition number of $\Sigma_\epsilon$ be $\kappa(\Sigma_\epsilon):=\frac{\lambda_{\max}(\Sigma_\epsilon)}{\lambda_{\min}(\Sigma_\epsilon)}$. Then,
\begin{equation}
c_1D\sqrt{\kappa(\Sigma_\epsilon)}\leq T(\theta)\leq c_2D\kappa(\Sigma_\epsilon)
\end{equation}
at the global minimum of \eqref{eq: free energy}, where $c_1,c_2$ are constants depending on $V\sqrt{\Sigma_x}$ and $d_x,d_y$.
\label{cor1}
\end{corollary}
Corollary \ref{cor1} suggests that the sharpness mainly depends on the imbalance of noise in the labels, and does not depend much on the depth. As an illustrative special case, we can choose $\Sigma_x=I$, $V=I$ ($d_x=d_y=d$), which gives $V'=\sqrt{\Sigma_\epsilon}=LSL^T$, and thus
\begin{equation}
\begin{aligned}
T(\theta)&=(D-1)(\Tr[\Sigma_\epsilon^{1/2}])^{(D-2)/D}(d\Tr[\Sigma_\epsilon])^{1/D}\Tr[\Sigma_\epsilon^{-1/2}]+d^{(D+1)/D}\Tr[\Sigma_\epsilon^{1/2}]^{2(D-1)/D}\Tr[\Sigma_\epsilon]^{-(D-1)/D},
\end{aligned}
\end{equation}
which can be arbitrarily large. On the other hand, we have $\min T(\theta)=Dd^{(D+2)/D}$, which does not depend on $\Sigma_\epsilon$. Thus, the imbalance of the noise spectrum can lead to arbitrarily high sharpness.  If we further choose $\Sigma_\epsilon=\text{diag}(1,\sigma^2)$, we have
\begin{equation}
\begin{aligned}
T(\theta)&=(D-1)2^{1/D}(1+\sigma)^{(D-2)/D}(1+\sigma^2)^{1/D}(1+1/\sigma)+2^{(D+1)/D}(1+\sigma)^{2(D-1)/D}(1+\sigma^2)^{-(D-1)/D}.
\end{aligned}
\end{equation}
One can check that it is symmetric in terms of $\sigma\to\sigma^{-1}$ and has a unique minimum at $\sigma=1$. We approximately have $T_{\min}\approx D$ and
\begin{equation}
\frac{T(\theta)}{T_{\min}}\approx \max(\sigma,1/\sigma)
\end{equation}
for large $D$ and very imbalanced noise in the labels. Thus, the gap between the sharpness of SGD and the minimal sharpness is approximately proportional to the imbalance of the noise in the labels and independent of the depth, as demonstrated by Corollary \ref{cor1}.

This could provide a theoretical explanation for the sharpening behavior observed in transformers. In token prediction, "stop words" (e.g., "the") have very low conditional entropy, while content words (nouns/verbs) have high entropy. This massive imbalance in token-level noise naturally drives learning toward sharper solutions. For the same reason, when there is class imbalance, the noise in the labels also tends to be imbalanced, and this will sharpen the solution. See Appendix \ref{app:power-law} for an example of the power-law covariance.

\paragraph{Recovery of Minimal Sharpness under Isotropic Noise.} 
\begin{corollary}
The global minimum is the flattest if and only if $\Sigma_\epsilon=I$.
\label{cor2}
\end{corollary}
The reason behind Corollary \ref{cor2} is that for deep linear networks with $\Sigma_\epsilon=I$, the sharpness is exactly proportional to the fluctuation (the second order derivative in \eqref{eq:Hessian} disappears for linear networks), and thus minimizing fluctuation is equivalent to minimizing sharpness.

\paragraph{Largest Eigenvalue Has a Similar Property} 
\begin{corollary}
Assume that $VV^T$ commutes with $\Sigma_\epsilon$ and $V^TV$ commutes with $\Sigma_x$. Then,
\begin{equation}
c_1D\sqrt{\kappa(\Sigma_\epsilon)}\leq \lambda_{\max}(\mathbb{E}[H])\leq c_2D\kappa(\Sigma_\epsilon),
\end{equation}
where $H$ is the Hessian at the global minimum of \eqref{eq: free energy} and $c_1,c_2$ are constants depending on $V$, $\Sigma_x$ and $d_y$.
\label{cor4}
\end{corollary}

Corollary \ref{cor4} suggests that in a special case where all the matrices are aligned, the maximal eigenvalue of the Hessian at the global minimum of \eqref{eq: free energy} is mainly determined by the noise imbalance and not by the depth. We expect that this should be true in general. For an illustrative example, we can choose $V=\Sigma_x=I$, $\Sigma_\epsilon=\text{diag}(1,\sigma^2)$. Then, we can obtain
\begin{equation}
\begin{aligned}
&\lambda_{\max}(\mathbb{E}[H])=(D-2)2^{1/D}(1+\sigma)^{-2/D}(1+\sigma^2)^{1/D}\max(1,\sigma^{-1})\max(1,\sigma)\\&\qquad+2^{1/D}(1+\sigma)^{(D-2)/D}(1+\sigma^2)^{-(D-1)/D}\max(1,\sigma)+2^{-(D-1)/D}(1+\sigma)^{(D-2)/D}(1+\sigma^2)^{1/D}\max(1,\sigma^{-1}).
\end{aligned}
\end{equation}
according to Theorem \ref{theo:max_eig}. One can check that it is symmetric in terms of $\sigma\to\sigma^{-1}$ and reaches its minimum at $\sigma=1$, which is the maximal eigenvalue of the minimal sharpness solution. We also have
\begin{equation}
\lambda_{\max}(\mathbb{E}[H])\approx D\max(\sigma,\sigma^{-1})
\end{equation}
for large $D$ and $\sigma\gg1$ or $\sigma^{-1}\gg1$. Therefore, the maximal eigenvalue scales similarly to the trace of the Hessian and grows as the noise becomes imbalanced.

\begin{wrapfigure}{r}{0.4\linewidth}
\centering
    \includegraphics[width=\linewidth]{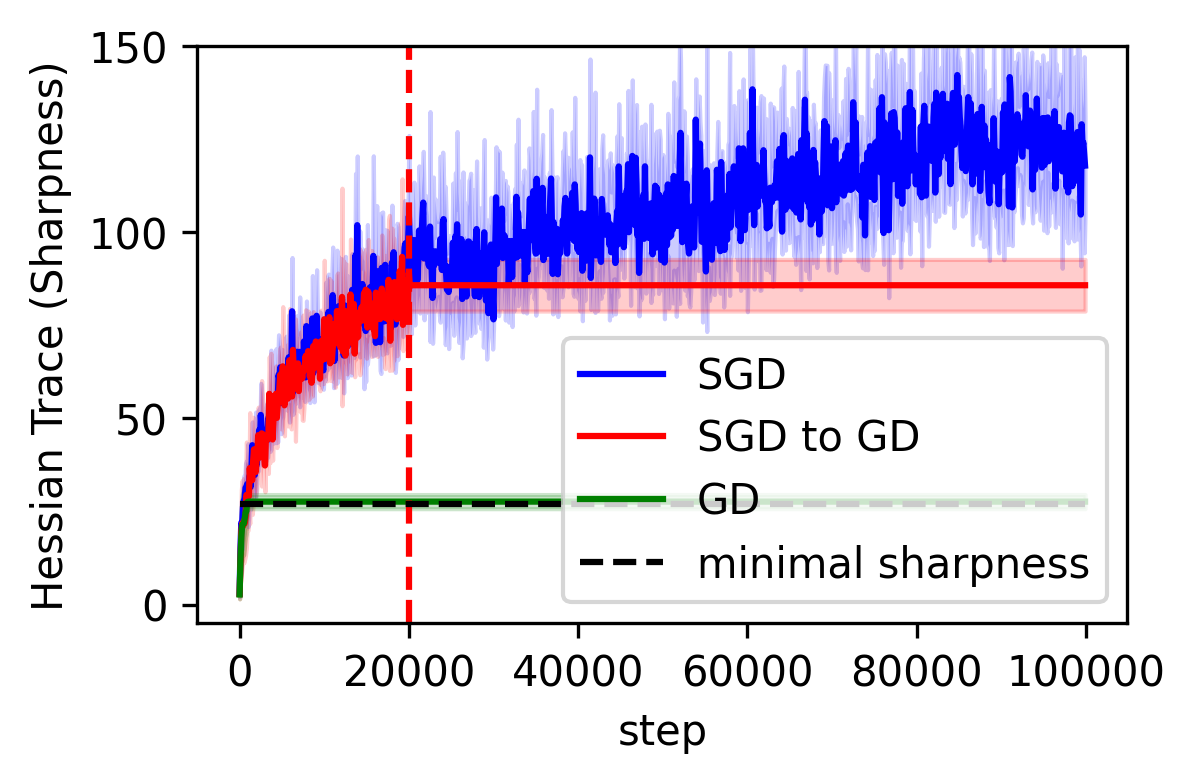}
    \caption{\small Progressive sharpening with non-isotropic noise in the labels is unique for SGD. The red dotted line denotes the time when we switch to GD.}
    \label{fig:GD}
\end{wrapfigure}

\paragraph{Progressive sharpening or flattening.} 
Theorem \ref{theo:main} suggests that SGD converges to a sharpness level that is independent of the initialization. Consequently, progressive sharpening (at least on average and throughout the trajectory) occurs if the initialization is sufficiently small, while progressive flattening occurs if it is large. In Proposition \ref{prop:initial_sharpness}, we characterize the initial sharpness and observe that it remains independent of the labels for standard initialization schemes. Comparing Proposition \ref{prop:initial_sharpness} with Theorem~\ref{theo:main} shows many mechanisms of progressive sharpening. Noting that the initial sharpness is independent of the label or noise distribution, while the final sharpness is label and noise-dependent, one finds a simple and straightforward mechanism for introducing flatening and sharpening: for a fixed initialization, rescaling the labels $y$ to be smaller causes a flattening, while rescaling $y$ to be larger can cause sharpening. Again, this showcases the key message: SGD intrinsically may not have a preference for sharpness, and it is the data distribution that changes, if not determines, the sharpness.

\paragraph{This is a \textit{new} mechanism of progressive sharpening.} The previous known mechanisms were about gradient flow and full batch gradient descent, as we mentioned in Section \ref{sec:sharp_or_flat}. The progressive sharpening that our theory discovers is not. Indeed, it is about noise covariance, and can not be seen in a deterministic training pipeline, as in the case of \citep{xing_walk_2018,jastrzebski_relation_2019,jastrzebski_break-even_2020,cohen2021gradient,cohen2022adaptive,andreyev2024edge}. In fact, the second term of \eqref{eq: free energy} is always zero at any local minimum for GD, and does not lead to progressive sharpening. This is demonstrated in Figure \ref{fig:GD}. Moreover, our progressive sharpening mechanism is not linked to the edge of stochastic stability and instability in general; it can be observed even with vanishing learning rates.



\paragraph{Experiments.} To empirically validate our theory, we perform experiments on both deep linear networks and several nonlinear architectures (MLP, RNN, and Transformer) across regression and classification tasks. These results, shown in Figures \ref{fig:label_noise} and \ref{fig:sharpness_condition_number}, demonstrate that the observed phenomena are not unique to linear models; specifically, imbalanced label noise induces sharper solutions as predicted by Corollary \ref{cor1}. Furthermore, the left side of Figure \ref{fig:label_noise} verifies Corollary \ref{cor2}, where linear networks converge to the minimal sharpness solution under isotropic label noise.

In Figure \ref{fig:sharpness_condition_number}, we observe that sharpness increases with the non-isotropy of label noise. The plateau at high condition numbers is a consequence of the finite-horizon effect, where highly imbalanced noise necessitates a very long convergence time to reach the sharpest minima. To isolate the role of stochasticity, Figure \ref{fig:GD} presents a direct intervention: when switching to GD in the linear network setting, progressive sharpening disappears. This identifies the mechanism as unique to SGD, aligning with \eqref{eq: free energy} where the second term vanishes for GD at local minima, exerting no implicit bias. Comprehensive details and additional experiments on eigenvalues, condition numbers, and Transformers can be found in Appendix \ref{app:exp}.

\section{General Theory}\label{sec: general theory}
As discussed, while many studies suggest SGD promotes flatness, recent observations of EoS behaviors indicate that SGD can also converge to remarkably "sharp" solutions. Our theory is that this difference in tendency arises from a subtle distinction between sharpness and fluctuation. In this section, we look at this problem from a general perspective and emphasize how our result leads to improved understanding here.

Consider the Mean Squared Error (MSE) loss\footnote{Similarly, one can extend the discussion to a generic loss function by looking at the imbalance in the spectrum of the label Hessian: $\nabla_{y}^2 \ell$.}
\begin{equation}
\ell(x,\theta):=\frac{1}{2}||f(x,\theta)-y(x)||^2.
\end{equation}
The Hessian is
\begin{equation}
\begin{aligned}
&H(\theta):=\mathbb{E}_x\nabla^2\ell(x,\theta)= \mathbb{E}_x\left(\nabla f(\theta,x)\nabla f(\theta,x)^T+ \sum_{j=1}^{d_y} [e(x,\theta)]_j \nabla^2_\theta [f(x, \theta)]_j \right),
\end{aligned}
\label{eq:Hessian}
\end{equation}
where $e(x,\theta):=y(x)-f(\theta,x)$ is the residual error. Near a global minimum, $e(x, \theta)$ becomes negligible, allowing us to approximate the sharpness (quantified by the trace of the Hessian) as
\begin{equation}
\Tr H(\theta)\approx\mathbb{E}_x||\nabla f(\theta,x)||^2,
\end{equation}
Such an approximation is well known in existing literature \cite{wu2023implicit,macdonald2023progressive}. In contrast, rather than the sharpness, recent works showed that SGD minimizes the following gradient fluctuation:
\begin{equation}
\mathbb{E}_x||\nabla\ell(x,\theta)||^2=\mathbb{E}_x||\nabla f(\theta,x)^Te(x,\theta)||^2,
\end{equation}
Assuming that the model is well-trained, the residual $e(x, \theta)$ primarily reflects noise in the labels.  If we treat this noise as a random variable with covariance $\Sigma_\epsilon$ independent of $x$, the fluctuation becomes:
\begin{equation}
\mathbb{E}_x||\nabla\ell(x,\theta)||^2=\mathbb{E}_x[\nabla f(x,\theta)^T\Sigma_\epsilon \nabla f(x,\theta)],
\label{eq:S}
\end{equation}
which is linked but not identical to the sharpness.
\begin{theorem}
Assuming that the residue $e(x,\theta):=y(x)-f(\theta,x)$ is a random variable with zero mean, variance $\Sigma_\epsilon$, and independent of the dataset. Then,
    \begin{equation}
        \lambda_{\min}(\Sigma_\epsilon)S\leq T \leq \lambda_{\max}(\Sigma_\epsilon)S,
    \end{equation}
where $T$ is the sharpness and $S$ is the gradient fluctuation.
\end{theorem}
\begin{proof}
By \eqref{eq:Hessian} we have
\begin{equation}
T:=\Tr H(\theta)=\mathbb{E}_x||\nabla f(\theta,x)||^2,
\end{equation}
where we use the assumption that $e(x,\theta)$ is independent of $x$ and has zero mean. By \eqref{eq:S} we have
\begin{equation}
S:=\mathbb{E}_x||\nabla\ell(x,\theta)||^2=\Tr[\Sigma_\epsilon\mathbb{E}_x[\nabla f(x,\theta)\nabla f(x,\theta)^T]].
\end{equation}
The proof is complete by combining them.
\end{proof}

This inequality makes it precise when minimizing fluctuation is not the same as minimizing the flatness. Crucially, the fluctuation is only a good approximation of the sharpness when the noise in the labels is isotropic. When $\frac{\lambda_{\max}(\Sigma_\epsilon)}{\lambda_{\min}(\Sigma_\epsilon)}$ is large, the bound becomes very loose and minimizing fluctuation no longer leads to minimizing flatness. 


\section{Related Works}


\paragraph{Existing minimal models of sharpness and edge of stability.}
Following the initial empirical discovery of progressive sharpening and the Edge of Stability (EoS) effect in GD and SGD \citep{xing_walk_2018,jastrzebski_relation_2019,jastrzebski_break-even_2020,cohen2021gradient,andreyev2024edge,cohen_understanding_2024}, considerable effort has been devoted to analyzing these phenomena through minimal models. \cite{agarwala2022second} analytically demonstrated progressive sharpening in a quadratic model as a direct consequence of EoS. \cite{li_analyzing_2022} showed progressive sharpening on a shallow network with tanh-activation. \cite{damian2022self} attributed these effects to the influence of cubic terms in the Taylor expansion of the loss function. 
Our work departs from these studies by demonstrating that such phenomena can arise from the intrinsic stochasticity of SGD, rather than being solely a function of the step size.
Importantly, \cite{ahn2023learning} proved the existence of oscillations at EoS in two-layer ReLU networks and highlighted its benefits for generalization, while \cite{beneventano2025gradient,ghosh_learning_2025} argued that oscillations in linear models implicitly bias GD toward flatter solutions. Recent work investigated the convergence and other implicit regularization aspects of \textsc{EoS} regime, e.g., see
\citep{arora_understanding_2022, wang_large_2022,ahn_understanding_2022,zhu_understanding_2023,lyu_understanding_2023}. 

\paragraph{Deep linear networks.} The key object of our study is the deep linear network model (also known as deep matrix factorization). 
Deep linear networks serve as essential proxies for understanding the training dynamics of complex architectures. However, most existing literature focuses on their behavior under gradient descent or gradient flow \citep{baldi_neural_1989,saxe2013exact,gunasekar_implicit_2017,arora2018convergence, eftekhari2020training, bah2022learning,du2019width, gidel_implicit_2019,tarmoun_understanding_2021,xu2023linear, nguegnang2024convergence}, where some of them suggest that GD possesses an implicit bias toward flatter minima in parameter space in these models \citep{marion2024deep,beneventano2025gradient,ghosh_learning_2025}. While some studies have also characterized their loss landscapes—showing that all local minima are global \citep{kawaguchi2016deep, laurent2018deep}---the role of noise remains less explored. Although \cite{ziyin2022exact} analyzed the global minima of deep linear networks with weight decay, our analysis specifically addresses the implicit bias induced by the stochasticity of SGD. Other closely related works are \cite{pesme2021implicit,even2023s,beneventano_how_2024}, which analyze diagonal neural networks (multiple dynamics with 1-dimensional label covariances), and show that the stochasticity is beneficial. 

\paragraph{Effect of noise in the labels.} There are a lot of studies on the implicit bias of noise in the labels. It is shown that the noise in the labels has an implicit regularization effect that biases the model to minima of weaker noise \cite{blanc2020implicit,li2021happens,haochen2021shape} or flatter minimum
\cite{damian2021label}. They also show that such a regularization effect is beneficial for generalization \cite{vivien2022label,huang2025does}. 
Importantly, all these work allow only 1-dimensional output. 1-dimensional matrices always have condition number 1. Our work, in contrast, shows that the noisy labeling can also bias the model to sharper solutions, dependent on the condition number. 

\section{Conclusion}
In this work, we studied in depth how and when SGD training can have a flatness or sharpness-seeking behavior. We provide an exact analytical solution for the dynamics of deep linear networks under the minimal gradient fluctuation constraint. In this model, we demonstrate that the converged sharpness is unique and independent of initialization. A key insight is that noise in the labels imbalance is a primary driver of the sharpening phenomenon -- a conclusion further validated through experiments on realistic datasets across model types. This is a novel form of progressive sharpening, of a different nature from the ones previously discussed in the literature. 

There are various limitations, each of which is worthwhile to be addressed by future work. First of all, our theory is effective in nature and does not directly reveal the dynamical aspects of training. A valuable open problem is thus to study the dynamics of reaching sharper and flatter solutions. Secondly, our theory only studies the leading order effect in $\eta$ during training, which tends to suppress fluctuation. This essentially corresponds to the small learning rate regime of training (but not as small as what is needed for gradient flow), and it is unclear what role higher-order terms in the entropic loss could play. Thus, understanding the flattening and sharpening effects with larger learning rates remains open. For example, it could well be the case that at a very large learning rate, the sharpness-seeking behavior becomes strongly suppressed again. For a similar reason, within our framework, only progressive sharpening can be explained, but not EoS. This is because to derive the entropic loss \eqref{eq: free energy}, one needs to essentially assume that the learning is somewhat stable \cite{smith2021origin}. Solving these problems will certainly advance our understanding of deep learning optimization.

\section*{Acknowledgment}

The authors thank Prof. Tomaso Poggio for discussion. ILC acknowledges support in part from the Institute for Artificial Intelligence and Fundamental Interactions (IAIFI) through NSF Grant No. PHY-2019786.

\bibliographystyle{alpha}
\bibliography{main}

@article{roberts_sgd_2021,
	title = {{SGD} {Implicitly} {Regularizes} {Generalization} {Error}},
	url = {http://arxiv.org/abs/2104.04874},
	urldate = {2021-08-05},
	journal = {arXiv:2104.04874 [cs, stat]},
	author = {Roberts, Daniel A.},
	month = apr,
	year = {2021},
	note = {arXiv:2104.04874},
}

@misc{mulayoff_exact_2024,
      title={Exact Mean Square Linear Stability Analysis for SGD}, 
      author={Rotem Mulayoff and Tomer Michaeli},
      year={2024},
      eprint={2306.07850},
      archivePrefix={arXiv},
      primaryClass={cs.LG},
      url={https://arxiv.org/abs/2306.07850}, 
}

@misc{xing_walk_2018,
	title = {A {Walk} with {SGD}},
	url = {http://arxiv.org/abs/1802.08770},
	language = {en},
	urldate = {2024-10-01},
	publisher = {arXiv},
	author = {Xing, Chen and Arpit, Devansh and Tsirigotis, Christos and Bengio, Yoshua},
	month = may,
	year = {2018},
	note = {arXiv:1802.08770 [cs, stat]},
}

@misc{jastrzebski_relation_2019,
	title = {On the {Relation} {Between} the {Sharpest} {Directions} of {DNN} {Loss} and the {SGD} {Step} {Length}},
	url = {http://arxiv.org/abs/1807.05031},
	doi = {10.48550/arXiv.1807.05031},
	urldate = {2025-02-04},
	publisher = {arXiv},
	author = {Jastrz\k{e}bski, Stanis\l{}aw and Kenton, Zachary and Ballas, Nicolas and Fischer, Asja and Bengio, Yoshua and Storkey, Amos},
	month = dec,
	year = {2019},
	note = {arXiv:1807.05031 [stat]},
}

@article{jastrzebski_break-even_2020,
	title = {The {Break}-{Even} {Point} on {Optimization} {Trajectories} of {Deep} {Neural} {Networks}},
	url = {http://arxiv.org/abs/2002.09572},
	language = {en},
	urldate = {2021-10-06},
	journal = {arXiv:2002.09572 [cs, stat]},
	author = {Jastrz\k{e}bski, Stanis\l{}aw and Szymczak, Maciej and Fort, Stanislav and Arpit, Devansh and Tabor, Jacek and Cho, Kyunghyun and Geras, Krzysztof},
	month = feb,
	year = {2020},
	note = {arXiv:2002.09572},
}

@article{ziyin2025neural,
  title={Neural Thermodynamics: Entropic Forces in Deep and Universal Representation Learning},
  author={Ziyin, Liu and Xu, Yizhou and Chuang, Isaac},
  journal={Advances in Neural Information Processing Systems},
  year={2025}
}

@inproceedings{li2021happens,
  title={What Happens after SGD Reaches Zero Loss?--A Mathematical Framework},
  author={Li, Zhiyuan and Wang, Tianhao and Arora, Sanjeev},
  booktitle={International Conference on Learning Representations},
  year={2021}
}

@article{wu2022alignment,
  title={The alignment property of {SGD} noise and how it helps select flat minima: A stability analysis},
  author={Wu, Lei and Wang, Mingze and Su, Weijie},
  journal={Advances in Neural Information Processing Systems},
  volume={35},
  pages={4680--4693},
  year={2022}
}

@ARTICLE{Dinh_SharpMinima,
   author = {{Dinh}, L. and {Pascanu}, R. and {Bengio}, S. and {Bengio}, Y.
	},
    title = "{Sharp Minima Can Generalize For Deep Nets}",
  journal = {ArXiv e-prints},
archivePrefix = "arXiv",
   eprint = {1703.04933},
 primaryClass = "cs.LG",
     year = 2017,
    month = mar,
   adsurl = {http://adsabs.harvard.edu/abs/2017arXiv170304933D}
}

@inproceedings{ziyin2024parameter,
  title={Parameter symmetry and noise equilibrium of stochastic gradient descent},
  author={Ziyin, Liu and Wang, Mingze and Li, Hongchao and Wu, Lei},
  booktitle={The Thirty-eighth Annual Conference on Neural Information Processing Systems},
  year={2024}
}

@article{keskar2016large,
  title={On large-batch training for deep learning: Generalization gap and sharp minima},
  author={Keskar, Nitish Shirish and Mudigere, Dheevatsa and Nocedal, Jorge and Smelyanskiy, Mikhail and Tang, Ping Tak Peter},
  journal={arXiv preprint arXiv:1609.04836},
  year={2016}
}

@article{wu2018sgd,
  title={How sgd selects the global minima in over-parameterized learning: A dynamical stability perspective},
  author={Wu, Lei and Ma, Chao and others},
  journal={Advances in Neural Information Processing Systems},
  volume={31},
  year={2018}
}

@article{xie2020diffusion,
  title={A diffusion theory for deep learning dynamics: Stochastic gradient descent exponentially favors flat minima},
  author={Xie, Zeke and Sato, Issei and Sugiyama, Masashi},
  journal={arXiv preprint arXiv:2002.03495},
  year={2020}
}

@inproceedings{wu2023implicit,
  title={The implicit regularization of dynamical stability in stochastic gradient descent},
  author={Wu, Lei and Su, Weijie J},
  booktitle={International Conference on Machine Learning},
  pages={37656--37684},
  year={2023},
  organization={PMLR}
}

@article{cohen2021gradient,
  title={Gradient descent on neural networks typically occurs at the edge of stability},
  author={Cohen, Jeremy M and Kaur, Simran and Li, Yuanzhi and Kolter, J Zico and Talwalkar, Ameet},
  journal={arXiv preprint arXiv:2103.00065},
  year={2021}
}

@article{agarwala2022second,
  title={Second-order regression models exhibit progressive sharpening to the edge of stability},
  author={Agarwala, Atish and Pedregosa, Fabian and Pennington, Jeffrey},
  journal={arXiv preprint arXiv:2210.04860},
  year={2022}
}

@article{hochreiter1997flat,
  title={Flat minima},
  author={Hochreiter, Sepp and Schmidhuber, J{\"u}rgen},
  journal={Neural computation},
  volume={9},
  number={1},
  pages={1--42},
  year={1997},
  publisher={MIT Press One Rogers Street, Cambridge, MA 02142-1209, USA journals-info~…}
}

@article{foret2020sharpness,
  title={Sharpness-aware minimization for efficiently improving generalization},
  author={Foret, Pierre and Kleiner, Ariel and Mobahi, Hossein and Neyshabur, Behnam},
  journal={arXiv preprint arXiv:2010.01412},
  year={2020}
}

@article{zhu2018anisotropic,
  title={The anisotropic noise in stochastic gradient descent: Its behavior of escaping from sharp minima and regularization effects},
  author={Zhu, Zhanxing and Wu, Jingfeng and Yu, Bing and Wu, Lei and Ma, Jinwen},
  journal={arXiv preprint arXiv:1803.00195},
  year={2018}
}

@article{smith2021origin,
  title={On the origin of implicit regularization in stochastic gradient descent},
  author={Smith, Samuel L and Dherin, Benoit and Barrett, David GT and De, Soham},
  journal={arXiv preprint arXiv:2101.12176},
  year={2021}
}

@article{macdonald2023progressive,
  title={On progressive sharpening, flat minima and generalisation},
  author={MacDonald, Lachlan Ewen and Valmadre, Jack and Lucey, Simon},
  journal={arXiv preprint arXiv:2305.14683},
  year={2023}
}

@article{cohen2022adaptive,
  title={Adaptive gradient methods at the edge of stability},
  author={Cohen, Jeremy M and Ghorbani, Behrooz and Krishnan, Shankar and Agarwal, Naman and Medapati, Sourabh and Badura, Michal and Suo, Daniel and Cardoze, David and Nado, Zachary and Dahl, George E and others},
  journal={arXiv preprint arXiv:2207.14484},
  year={2022}
}

@article{damian2022self,
  title={Self-stabilization: The implicit bias of gradient descent at the edge of stability},
  author={Damian, Alex and Nichani, Eshaan and Lee, Jason D},
  journal={arXiv preprint arXiv:2209.15594},
  year={2022}
}

@article{beneventano2025gradient,
  title={Gradient Descent Converges Linearly to Flatter Minima than Gradient Flow in Shallow Linear Networks},
  author={Beneventano, Pierfrancesco and Woodworth, Blake},
  journal={arXiv preprint arXiv:2501.09137},
  year={2025}
}

@article{ahn2023learning,
  title={Learning threshold neurons via edge of stability},
  author={Ahn, Kwangjun and Bubeck, S{\'e}bastien and Chewi, Sinho and Lee, Yin Tat and Suarez, Felipe and Zhang, Yi},
  journal={Advances in Neural Information Processing Systems},
  volume={36},
  pages={19540--19569},
  year={2023}
}

@article{andreyev2024edge,
  title={Edge of stochastic stability: Revisiting the edge of stability for sgd},
  author={Andreyev, Arseniy and Beneventano, Pierfrancesco},
  journal={arXiv preprint arXiv:2412.20553},
  year={2024}
}

@article{saxe2013exact,
  title={Exact solutions to the nonlinear dynamics of learning in deep linear neural networks},
  author={Saxe, Andrew M and McClelland, James L and Ganguli, Surya},
  journal={arXiv preprint arXiv:1312.6120},
  year={2013}
}

@article{arora2018convergence,
  title={A convergence analysis of gradient descent for deep linear neural networks},
  author={Arora, Sanjeev and Cohen, Nadav and Golowich, Noah and Hu, Wei},
  journal={arXiv preprint arXiv:1810.02281},
  year={2018}
}

@inproceedings{xu2023linear,
  title={Linear convergence of gradient descent for finite width over-parametrized linear networks with general initialization},
  author={Xu, Ziqing and Min, Hancheng and Tarmoun, Salma and Mallada, Enrique and Vidal, Ren{\'e}},
  booktitle={International Conference on Artificial Intelligence and Statistics},
  pages={2262--2284},
  year={2023},
  organization={PMLR}
}

@article{nguegnang2024convergence,
  title={Convergence of gradient descent for learning linear neural networks},
  author={Nguegnang, Gabin Maxime and Rauhut, Holger and Terstiege, Ulrich},
  journal={Advances in Continuous and Discrete Models},
  volume={2024},
  number={1},
  pages={23},
  year={2024},
  publisher={Springer}
}

@article{ziyin2022exact,
  title={Exact solutions of a deep linear network},
  author={Ziyin, Liu and Li, Botao and Meng, Xiangming},
  journal={Advances in Neural Information Processing Systems},
  volume={35},
  pages={24446--24458},
  year={2022}
}

@article{marion2024deep,
  title={Deep linear networks for regression are implicitly regularized towards flat minima},
  author={Marion, Pierre and Chizat, L{\'e}na{\"\i}c},
  journal={Advances in Neural Information Processing Systems},
  volume={37},
  pages={76848--76900},
  year={2024}
}

@article{kawaguchi2016deep,
  title={Deep learning without poor local minima},
  author={Kawaguchi, Kenji},
  journal={Advances in neural information processing systems},
  volume={29},
  year={2016}
}

@inproceedings{laurent2018deep,
  title={Deep linear networks with arbitrary loss: All local minima are global},
  author={Laurent, Thomas and Brecht, James},
  booktitle={International conference on machine learning},
  pages={2902--2907},
  year={2018},
  organization={PMLR}
}

@article{pesme2021implicit,
  title={Implicit bias of sgd for diagonal linear networks: a provable benefit of stochasticity},
  author={Pesme, Scott and Pillaud-Vivien, Loucas and Flammarion, Nicolas},
  journal={Advances in Neural Information Processing Systems},
  volume={34},
  pages={29218--29230},
  year={2021}
}

@article{damian2021label,
  title={Label noise sgd provably prefers flat global minimizers},
  author={Damian, Alex and Ma, Tengyu and Lee, Jason D},
  journal={Advances in Neural Information Processing Systems},
  volume={34},
  pages={27449--27461},
  year={2021}
}

@inproceedings{haochen2021shape,
  title={Shape matters: Understanding the implicit bias of the noise covariance},
  author={HaoChen, Jeff Z and Wei, Colin and Lee, Jason and Ma, Tengyu},
  booktitle={Conference on Learning Theory},
  pages={2315--2357},
  year={2021},
  organization={PMLR}
}

@inproceedings{blanc2020implicit,
  title={Implicit regularization for deep neural networks driven by an ornstein-uhlenbeck like process},
  author={Blanc, Guy and Gupta, Neha and Valiant, Gregory and Valiant, Paul},
  booktitle={Conference on learning theory},
  pages={483--513},
  year={2020},
  organization={PMLR}
}

@inproceedings{vivien2022label,
  title={Label noise (stochastic) gradient descent implicitly solves the lasso for quadratic parametrisation},
  author={Pillaud-Vivien, Loucas and Reygner, Julien and Flammarion, Nicolas},
  booktitle={Conference on Learning Theory},
  pages={2127--2159},
  year={2022},
  organization={PMLR}
}

@article{huang2025does,
  title={How Does Label Noise Gradient Descent Improve Generalization in the Low SNR Regime?},
  author={Huang, Wei and Han, Andi and Song, Yujin and Chen, Yilan and Wu, Denny and Zou, Difan and Suzuki, Taiji},
  journal={arXiv preprint arXiv:2510.17526},
  year={2025}
}

@article{beneventano2023trajectories,
  title={On the trajectories of sgd without replacement},
  author={Beneventano, Pierfrancesco},
  journal={arXiv preprint arXiv:2312.16143},
  year={2023}
}

@article{beneventano_how_2024,
	title = {How {Neural} {Networks} {Learn} the {Support} is an {Implicit} {Regularization} {Effect} of {SGD}},
	url = {http://arxiv.org/abs/2406.11110},
        doi = {10.48550/arXiv.2406.11110},
        publisher = {arXiv},
	language = {en},
	journal = {arXiv:2406.11110 [cs, math, stat]},
	author = {Beneventano, Pierfrancesco and Pinto, Andrea and Poggio, Tomaso},
	month = jun,
	year = {2024},
        note = {arXiv:2406.11110 [cs, math, stat]},
}

@article{kamber2025sharpness,
  title={Sharpness of Minima in Deep Matrix Factorization: Exact Expressions},
  author={Kamber, Anil and Parhi, Rahul},
  journal={arXiv preprint arXiv:2509.25783},
  year={2025}
}

@article{even2023s,
  title={(S) GD over Diagonal Linear Networks: Implicit bias, Large Stepsizes and Edge of Stability},
  author={Even, Mathieu and Pesme, Scott and Gunasekar, Suriya and Flammarion, Nicolas},
  journal={Advances in Neural Information Processing Systems},
  volume={36},
  pages={29406--29448},
  year={2023}
}

@inproceedings{gidel_implicit_2019,
	title = {Implicit {Regularization} of {Discrete} {Gradient} {Dynamics} in {Linear} {Neural} {Networks}},
	volume = {32},
	url = {https://proceedings.neurips.cc/paper_files/paper/2019/hash/f39ae9ff3a81f499230c4126e01f421b-Abstract.html},
	urldate = {2024-04-22},
	booktitle = {Advances in {Neural} {Information} {Processing} {Systems}},
	publisher = {Curran Associates, Inc.},
	author = {Gidel, Gauthier and Bach, Francis and Lacoste-Julien, Simon},
	year = {2019},
}

@inproceedings{tarmoun_understanding_2021,
	title = {Understanding the dynamics of gradient flow in overparameterized linear models},
	booktitle = {International {Conference} on {Machine} {Learning}},
	publisher = {PMLR},
	author = {Tarmoun, Salma and Franca, Guilherme and Haeffele, Benjamin D and Vidal, Rene},
	year = {2021},
	pages = {10153--10161},
}

@misc{ghosh_learning_2025,
      title={Learning Dynamics of Deep Linear Networks Beyond the Edge of Stability}, 
      author={Avrajit Ghosh and Soo Min Kwon and Rongrong Wang and Saiprasad Ravishankar and Qing Qu},
      year={2025},
      eprint={2502.20531},
      archivePrefix={arXiv},
      primaryClass={stat.ML},
      url={https://arxiv.org/abs/2502.20531}, 
}

@misc{cohen_understanding_2024,
	title = {Understanding {Optimization} in {Deep} {Learning} with {Central} {Flows}},
	url = {http://arxiv.org/abs/2410.24206},
	doi = {10.48550/arXiv.2410.24206},
	publisher = {arXiv},
	author = {Cohen, Jeremy M. and Damian, Alex and Talwalkar, Ameet and Kolter, Zico and Lee, Jason D.},
	month = oct,
	year = {2024},
	note = {arXiv:2410.24206},
}

@misc{arora_understanding_2022,
	title = {Understanding {Gradient} {Descent} on {Edge} of {Stability} in {Deep} {Learning}},
	url = {http://arxiv.org/abs/2205.09745},
	doi = {10.48550/arXiv.2205.09745},
	urldate = {2025-01-26},
	publisher = {arXiv},
	author = {Arora, Sanjeev and Li, Zhiyuan and Panigrahi, Abhishek},
	month = oct,
	year = {2022},
	note = {arXiv:2205.09745 [cs]},
}

@inproceedings{wang_large_2022,
title={Large Learning Rate Tames Homogeneity: Convergence and Balancing Effect},
author={Wang, Yuqing and Chen, Minshuo and Zhao, Tuo and Tao, Molei},
booktitle={International Conference on Learning Representations},
year={2022},
url={https://openreview.net/forum?id=3tbDrs77LJ5}
}

@inproceedings{ahn_understanding_2022,
	title = {Understanding the unstable convergence of gradient descent},
	url = {https://proceedings.mlr.press/v162/ahn22a.html},
	language = {en},
	urldate = {2024-10-02},
	booktitle = {Proceedings of the 39th {International} {Conference} on {Machine} {Learning}},
	author = {Ahn, Kwangjun and Zhang, Jingzhao and Sra, Suvrit},
	month = jun,
	year = {2022},
}

@misc{lyu_understanding_2023,
	title = {Understanding the {Generalization} {Benefit} of {Normalization} {Layers}: {Sharpness} {Reduction}},
	shorttitle = {Understanding the {Generalization} {Benefit} of {Normalization} {Layers}},
	url = {http://arxiv.org/abs/2206.07085},
	doi = {10.48550/arXiv.2206.07085},
	urldate = {2024-10-02},
	publisher = {arXiv},
	author = {Lyu, Kaifeng and Li, Zhiyuan and Arora, Sanjeev},
	month = jan,
	year = {2023},
	note = {arXiv:2206.07085 [cs]},
}

@misc{li_analyzing_2022,
	title = {Analyzing {Sharpness} along {GD} {Trajectory}: {Progressive} {Sharpening} and {Edge} of {Stability}},
	shorttitle = {Analyzing {Sharpness} along {GD} {Trajectory}},
	url = {http://arxiv.org/abs/2207.12678},
	doi = {10.48550/arXiv.2207.12678},
	author = {Li, Zhouzi and Wang, Zixuan and Li, Jian},
	month = nov,
	year = {2022},
	note = {arXiv:2207.12678},
}

@misc{zhu_understanding_2023,
	title = {{UNDERSTANDING} {EDGE}-{OF}-{STABILITY} {TRAINING} {DYNAMICS} {WITH} {A} {MINIMALIST} {EXAMPLE}},
	language = {en},
	author = {Zhu, Xingyu and Wang, Zixuan and Wang, Xiang and Zhou, Mo and Ge, Rong},
	year = {2023},
}

@article{baldi_neural_1989,
	title = {Neural networks and principal component analysis: {Learning} from examples without local minima},
	volume = {2},
	issn = {08936080},
	shorttitle = {Neural networks and principal component analysis},
	url = {https://linkinghub.elsevier.com/retrieve/pii/0893608089900142},
	doi = {10.1016/0893-6080(89)90014-2},
	language = {en},
	number = {1},
	urldate = {2023-10-25},
	journal = {Neural Networks},
	author = {Baldi, Pierre and Hornik, Kurt},
	month = jan,
	year = {1989},
	pages = {53--58},
}

@inproceedings{gunasekar_implicit_2017,
	title = {Implicit regularization in matrix factorization},
	booktitle = {Advances in {Neural} {Information} {Processing} {Systems}},
	author = {Gunasekar, Suriya and Woodworth, Blake E and Bhojanapalli, Srinadh and Neyshabur, Behnam and Srebro, Nati},
	year = {2017},
	pages = {6151--6159},
}

@inproceedings{du2019width,
  title={Width provably matters in optimization for deep linear neural networks},
  author={Du, Simon and Hu, Wei},
  booktitle={International Conference on Machine Learning},
  pages={1655--1664},
  year={2019},
  organization={PMLR}
}

@inproceedings{eftekhari2020training,
  title={Training linear neural networks: Non-local convergence and complexity results},
  author={Eftekhari, Armin},
  booktitle={International Conference on Machine Learning},
  pages={2836--2847},
  year={2020},
  organization={PMLR}
}

@article{bah2022learning,
  title={Learning deep linear neural networks: Riemannian gradient flows and convergence to global minimizers},
  author={Bah, Bubacarr and Rauhut, Holger and Terstiege, Ulrich and Westdickenberg, Michael},
  journal={Information and Inference: A Journal of the IMA},
  volume={11},
  number={1},
  pages={307--353},
  year={2022},
  publisher={Oxford University Press}
}

\appendix

\clearpage

\section{Theory}
\subsection{Proof of Lemma \ref{lemma:sharpness}}
\begin{proof}
By the definition of $A$-symmetry, the loss function satisfies the following invariance property for any $\lambda$:
\begin{equation}
\ell(x,y,e^{\lambda A}\theta)=\ell(x,y,\theta).
\end{equation}
To analyze the sharpness, we take the second-order derivative with respect to $\theta$ on both sides. Applying the chain rule, we obtain:
\begin{equation}
   e^{\lambda A} \nabla_{e^{\lambda A}\theta }^2 \ell (x,y, e^{\lambda A}\theta ) e^{\lambda A}  = \nabla^2 \ell(x,y, \theta),
\end{equation}
where we use the fact that $A$ is symmetric.

Using the linearity and cyclic property of the trace, we have:
\begin{equation}
T(e^{\lambda A}\theta):=\Tr[\mathbb{E}_{x,y}\nabla_{e^{\lambda A}\theta }^2 \ell (x,y, e^{\lambda A}\theta )]=\Tr[e^{-2\lambda A}\E_{x,y}\nabla^2\ell(x,y,\theta)].
\end{equation}
Let $A:=\sum_i\mu_in_in_i^T$. The sharpness can then be expressed as a weighted sum:
\begin{equation}
T(e^{\lambda A}\theta)=\sum_{i}e^{-2\lambda\mu_i}(n_i^T\E_{x,y}\nabla^2\ell(x,y,\theta)n_i).
\end{equation}
By the assumption $A \mathbb{E}{x,y} \nabla^2 \ell(x, y, \theta) \neq 0$, there must exist at least one index $i$ such that $\mu_i\neq0$ and  $n_i^T\E\nabla^2\ell(x,\theta)n_i\neq0$.  Consequently, the term $e^{-2\lambda \mu_i}$ will dominate as $\lambda$ tends to infinity with the appropriate sign. Specifically, we have $\lim_{\lambda\to+\infty}|T(e^{\lambda A}\theta)|=+\infty$ if there exists $\mu_i<0$, and $\lim_{\lambda\to-\infty}|T(e^{\lambda A}\theta)|=+\infty$ if there exists $\mu_i>0$.
\end{proof}

\subsection{Proof of Lemma \ref{lemma: minimal fluctuation}}
We first provide a formal statement of Lemma \ref{lemma: minimal fluctuation}.
\begin{lemma}
Let $f, S$ be continuous functions, and let $f$ be coercive (i.e., $\lim_{|\theta| \to \infty} f(\theta) = +\infty$). Define the perturbed objective function
\begin{equation}
F_\eta(\theta)=f(\theta)+\eta S(\theta).
\end{equation}
Let $\{\eta_k\}_{k=1}^\infty$ be a sequence of positive numbers such that $\eta_k \to 0$ as $k \to \infty$. Suppose
\begin{equation}
\theta_k\in\arg\min_\theta f(\theta)+\eta_k S(\theta).
\end{equation}
and $\theta_k\to\bar{\theta}$\footnote{As $f$ is coercive, $M$ is compact and thus there exists a subsequence of $\{\theta_k\}$ which converges. Here for simplicity we assume $\{\theta_k\}$ converges.} as $k \to \infty$. Then $\bar{\theta}$ minimizes $S$ over the set of minimizers of $f$, i.e., \begin{equation} \bar{\theta} \in \arg\min_{\theta \in M} S(\theta), \quad \text{where } M := \arg\min_{\theta} f(\theta). \end{equation} 
\end{lemma}

\begin{proof}
Since $f$ is coercive and continuous, the set of minimizers $M$ is non-empty and compact. Let $y \in M$ be any global minimizer of $f$. By the optimality of $\theta_k$ for $F_{\eta_k}$, we have:
\begin{equation}
f(\theta_k) + \eta_k S(\theta_k) \leq f(y) + \eta_k S(y).
\end{equation}
Rearranging the terms, we obtain:
\begin{equation}
f(\theta_k) - f(y) \leq \eta_k (S(y) - S(\theta_k))
\label{ineq:f,S}
\end{equation}
Since $y \in M$, we have $f(\theta_k) \geq f(y)$, which implies $f(\theta_k) - f(y) \geq 0$. Consequently, from \eqref{ineq:f,S} and the fact that $\eta_k > 0$, it follows that:
\begin{equation}
S(\theta_k) \leq S(y).
\label{eq:S_ineq}
\end{equation}
Taking the limit as $k \to \infty$ in \eqref{ineq:f,S}, and noting that $S$ is continuous (thus $S(\theta_k)$ remains bounded as $\theta_k \to \bar{\theta}$), the right-hand side vanishes:
\begin{equation}
0 \leq \lim_{k \to \infty} (f(\theta_k) - f(y)) \leq \lim_{k \to \infty} \eta_k (S(y) - S(\theta_k)) = 0.
\end{equation}
By the continuity of $f$, we have $f(\bar{\theta}) = f(y)$, which implies $\bar{\theta} \in M$.

Finally, we consider the inequality \eqref{eq:S_ineq}. Taking the limit $k \to \infty$ on both sides and utilizing the continuity of $S$, we obtain:
\begin{equation}
S(\bar{\theta}) \leq S(y) \quad \text{for all } y \in M.
\end{equation}
This inequality demonstrates that $\bar{\theta}$ minimizes $S$ over $M$, i.e., $\bar{\theta} \in \arg\min_{\theta \in M} S(\theta)$, which completes the proof.
\end{proof}

\subsection{Proof of Theorem \ref{theo:main}}
\begin{proof}
Step 1: Computing the Hessian Trace.

We begin by evaluating the trace of the Hessian with respect to the weight matrices $\{W_k\}_{k=1}^D$. For each layer $k$, let $A_k := W_D \cdots W_{k+1}$ and $B_k := W_{k-1} \cdots W_1$, with the conventions $A_D = I_{d_y}$ and $B_1 = I_{d_x}$.  The loss function for a single sample can be written as $\ell(x, y,\theta) = \frac{1}{2} \|y - A_k W_k B_k x\|^2$. The second-order partial derivative with respect to $W_k$ yields the following identity:
\begin{equation}
\Tr\frac{\partial^2\ell}{\partial W_k^2}=\Tr A_k^TA_k\Tr B_kxx^TB_k^T.
\end{equation}
By taking the expectation over the data distribution and summing over all layers $k=1, \dots, D$, the total sharpness $T(\theta)$ is given by:
\begin{equation}
T(\theta)=\sum_{k=1}^D\mathbb{E}_{x,y}\Tr\frac{\partial^2\ell}{\partial W_k^2}=\sum_{k=1}^D\Tr [A_k^TA_k]\Tr [B_k\Sigma_xB_k^T].
\label{eq:T-A-B}
\end{equation}
This establishes the structural form of the sharpness. The minimization of $T(\theta)$ as per \eqref{eq:minT} then follows directly from applying Lemma \ref{lemma:min_sharpness} to the summation in \eqref{eq:T-A-B}.

Step 2: Analysis of Sharpness at the Minimum of \eqref{eq: free energy}.

For the solutions from Lemma \ref{lemma:deep_linear}, the product matrices are decomposed as
\begin{equation}
A_k=\sqrt{\Sigma_{\epsilon}^{-1}}L\Sigma_D\cdots\Sigma_{k+1}U_{k}^T,\ B_k=U_{k-1}\Sigma_{k-1}\cdots\Sigma_1R\Sigma_x^{-1/2}.
\end{equation}
For the intermediate layers $k \in \{2, \dots, D-1\}$, the trace product simplifies to:
\begin{equation}
\begin{aligned}
\Tr [A_k^TA_k]\Tr [B_k\Sigma_xB_k^T]&=\Tr[\Sigma_\epsilon^{-1}L\Sigma^2_D\cdots\Sigma_{k+1}^2L^T]\Tr[\Sigma_1^2\cdots\Sigma_{k-1}^2]\\
&=\left(\Tr S\right)^{-2/D}(\Tr[\Sigma_\epsilon]\Tr[\Sigma_x])^{1/D}\Tr[\Sigma_\epsilon^{-1}LSL^T]\Tr[S]
\end{aligned}
\end{equation}
The boundary layers require separate treatment. For $k=1$ and $k=D$, we have:
\begin{equation}
\Tr [A_1^TA_1]\Tr [B_1\Sigma_xB_1^T]=\Tr[\Sigma_\epsilon^{-1}L\Sigma^2_D\cdots\Sigma_{2}^2L^T]\Tr[\Sigma_x]=\left(\Tr S\right)^{(D-2)/D}\frac{\Tr[\Sigma_\epsilon]^{1/D}}{\Tr[\Sigma_x]^{(D-1)/D}}\Tr[\Sigma_\epsilon^{-1}LSL^T]\Tr[\Sigma_x]
\end{equation}
and
\begin{equation}
\Tr [A_D^TA_D]\Tr [B_D\Sigma_xB_D^T]=\Tr[I_{d_y}]\Tr[\Sigma_1^2\cdots\Sigma_{D-1}^2]=d_y\left(\Tr S\right)^{(D-2)/D}\frac{\Tr[\Sigma_x]^{1/D}}{\Tr[\Sigma_\epsilon]^{(D-1)/D}}\Tr[S].
\end{equation}
Summing these contributions over $k=1, \dots, D$ yields the final expression for $T(\theta)$:
\begin{equation}
\begin{aligned}
T(\theta)&=\sum_{k=1}^D\Tr [A_k^TA_k]\Tr [B_k\Sigma_xB_k^T]\\&=(D-1)\left(\Tr S\right)^{(D-2)/D}(\Tr[\Sigma_\epsilon]\Tr[\Sigma_x])^{1/D}\Tr[\Sigma_\epsilon^{-1}LSL^T]+d_y\left(\Tr S\right)^{2(D-1)/D}\frac{\Tr[\Sigma_x]^{1/D}}{\Tr[\Sigma_\epsilon]^{(D-1)/D}},
\end{aligned}
\end{equation}
which proves \eqref{eq:sharpness}.

Step 3: Invariance of the Hessian Spectrum. 

Finally, we verify that the spectrum of the Hessian $H(\theta)$ is constant across the manifold of solutions defined in Lemma \ref{lemma:deep_linear}. Consider two sets of weight matrices $\theta = \{W_k\}$ and $\theta' = \{W_k'\}$. According to the lemma, any two such solutions are related by a sequence of orthogonal transformations $R_k \in O(d_k)$, such that $W_k' = R_k W_k R_{k-1}^T$. This transformation can be represented as $\theta' = \mathcal{M} \theta$, where $\mathcal{M}$ is a block-diagonal orthogonal matrix. Since the loss function is invariant under these internal rotations, the Hessian transforms as $H(\theta') = \mathcal{M} H(\theta) \mathcal{M}^T$. Because $\mathcal{M}$ is orthogonal, this is a similarity transformation, which implies that the eigenvalues (and thus the spectrum) of the Hessian remain invariant.\
\end{proof}

The following are technical lemmas for Theorem \ref{theo:main}.
\begin{lemma}
\label{lemma:min_sharpness}
Consider 
\begin{equation}
\begin{aligned}
T(W_1,\cdots,W_D)&=d_y\Tr[W_{D-1}\cdots W_1\Sigma W_1^T\cdots W_{D-1}^T]+\Tr[W_D^TW_D]\Tr[W_{D-2}\cdots W_1\Sigma W_1^T\cdots W_{D-2}^T]\\&\qquad+\cdots+\Tr[W_2^T\cdots W_D^TW_D\cdots W_2]\Tr[\Sigma],
\end{aligned}
\label{eq:sharpness_expression}
\end{equation}
where $d_y$ is the output dimension. Then, under the constraint $W_D\cdots W_1=V$, we have
\begin{equation}
\min T(W_1,\cdots,W_D)=D(\Tr\hat{S})^{2(D-1)/D}(d_y\Tr\Sigma)^{1/D},
\end{equation}
where $V\sqrt{\Sigma}=\hat{L}\hat{S}\hat{R}$ is the singular value decomposition. The minimum is achieved when
\begin{equation}
\begin{aligned}
&W_D=(\Tr S)^{-\frac{D-2}{2D}}d_y^{\frac{D-1}{2D}}(\Tr\Sigma)^{-\frac{1}{2D}}\hat{L}\hat{S}^{\frac{1}{2}}U_{D-1}^T,\\ &W_k=(\Tr S)^{\frac{1}{D}}(d_y\Tr\Sigma)^{-\frac{1}{2D}}U_{k}^TU_{k-1},\\ &W_1\sqrt{\Sigma}=(\Tr S)^{-\frac{D-2}{2D}}d_y^{-\frac{1}{2D}}(\Tr\Sigma)^{\frac{D-1}{2D}}U_1\hat{S}^{\frac{1}{2}}\hat{R},
\end{aligned}
\label{eq:min_sharpness_solution}
\end{equation}
for $k=2,\cdots,D-1$, where $\{U_i\}_{i=1}^{D-1}$ are arbitrary matrices satisfying $U_i^TU_i=I_{d\times d}$.
\end{lemma}
\begin{proof}
Step 1: Lower Bound.

We seek to minimize the sum of $D$ terms in \eqref{eq:sharpness_expression}. For brevity, let $W_1' = W_1 \Sigma^{1/2}$. By the AM-GM inequality, the sharpness $T$ is bounded below by the geometric mean of its components:
\begin{equation}
\begin{aligned}
T(W_1,\cdots,W_D)^D/D^D&\geq d_y||W_{D-1}\cdots W_1'||_F^2||W_D||_F^2||W_{D-2}\cdots W_1’||_F^2\cdots||W_D\cdots W_2||_F^2\Tr[\Sigma]\\
&=d_y\Tr[\Sigma]\Pi_{k=1}^{D-1}||W_D\cdots W_{k+1}||_F^2||W_k\cdots W_{1}'||_F^2
\end{aligned}
\end{equation}
Applying the matrix inequality
\begin{equation}
||A||_F||B||_F\geq||AB||_*
\label{eq:F-norm-ineq}
\end{equation}
where $||\cdot||_*$ is the nuclear norm (sum of singular values), to each pair in the product, we have:
\begin{equation} ||W_D \dots W_{k+1}||_F ||W_k \dots W_1'||_F \geq ||W_D \dots W_1'|| = ||V \Sigma^{1/2}||_* = \Tr\hat{S}. \end{equation}
Substituting this into the AM-GM result, we obtain the minimal sharpness:
\begin{equation}
T(W_1,\cdots,W_D)^D/D^D\geq d_y\Tr[\Sigma](\Tr\hat{S})^{2(D-1)}.
\end{equation}
Step 2: Conditions for Equality.

The minimum is achieved if and only if $A_k^TA_k=cB_kB_k^T$ for some $c\geq0$ and every term in \eqref{eq:sharpness_expression} is the same, which requires the following relations:
\begin{equation}
W_{k+1}^{\top} \frac{W_{k+2}^{\top} \cdots W_D^{\top} W_D \cdots W_{k+2}}{\Tr \left[W_{k+2}^{\top} \cdots W_D^{\top} W_D \cdots W_{k+2}\right]} W_{k+1}=W_k \frac{W_{k-1} \cdots W_1\Sigma W_1^{\top} \cdots W_{k-1}^{\top}}{\Tr \left[W_{k-1} \cdots W_1\Sigma W_1^{\top} \cdots W_{k-1}^{\top}\right]} W_k^{\top}
\label{eq:eqi}
\end{equation}
for $k=2,3,\cdots,D-2$ and
\begin{equation}
\begin{aligned}
&W_{2}^{\top} \frac{W_{3}^{\top} \cdots W_D^{\top} W_D\cdots W_{3}}{\Tr \left[W_{3}^{\top} \cdots W_D^{\top} W_D\cdots W_{3}\right]} W_{2}= \frac{W_1\Sigma W_1^{\top} }{\Tr \left[\Sigma\right]}, \\
&\frac{W_D^{\top} W_D}{d_y}=W_{D-1} \frac{W_{D-2} \cdots W_1\Sigma W_1^{\top} \cdots W_{D-2}^{\top}}{\Tr \left[W_{D-2} \cdots W_1\Sigma W_1^{\top} \cdots W_{D-2}^{\top}\right]} W_{D-1}^{\top}.
\label{eq:eq1-D}
\end{aligned}
\end{equation}

Step 3: Determining the Weights.

By Lemma 2 of \cite{ziyin2025neural}, we can decompose the matrices $W_1',W_2,\cdots,W_{D-1},W_D$ as
\begin{equation}
W_D=U_D\Sigma_DU_{D-1}^T,\ W_{D-1}=U_{D-1}\Sigma_{D-1}U_{D-2}^T,\cdots,\ W_1'=U_1\Sigma_1U_0,
\end{equation}
where $W_1':=W_1\sqrt{\Sigma}$. From $W_D\cdots W_1=V$ we have $U_D=\hat{L}$, $U_0=\hat{R}$ and
\begin{equation}
\Sigma_D\Sigma_{D-1}\cdots\Sigma_1=\hat{S}.
\label{eq:S'}
\end{equation}
We can assume $\Sigma_D,\cdots,\Sigma_1\in\mathbb{R}^{d\times d}$ without loss of generality because their ranks are the same by \eqref{eq:eqi}, \eqref{eq:eq1-D}. Then, \eqref{eq:eqi} gives
\begin{equation}
\frac{\Sigma_{i+1}^2\Sigma_{i+2}^2\cdots\Sigma_D^2}{\Tr [\Sigma_{i+2}^2\cdots\Sigma_D^2]}=\frac{\Sigma_1^2\cdots\Sigma_{i-1}^2\Sigma_i^2}{\Tr[\Sigma_1^2\cdots\Sigma_{i-1}^2]},
\end{equation}
and thus $\Sigma_i=cI_d$ for $i=2,3,\cdots,D-2$ and
\begin{equation}
\frac{\Sigma_1^2}{\Tr[\Sigma_1^2]}=\frac{\Sigma_D^2}{\Tr[\Sigma_D^2]}.
\label{eq:Sigma1D_1}
\end{equation}
\eqref{eq:eq1-D} gives
\begin{equation}
\frac{\Sigma_{2}^2\Sigma_{3}^2\cdots\Sigma_D^2}{\Tr [\Sigma_{3}^2\cdots\Sigma_D^2]}=\frac{\Sigma_1^2}{\Tr[\Sigma]},\ \frac{\Sigma_D^2}{d_y}=\frac{\Sigma_1^2\cdots\Sigma_{D-1}^2}{\Tr[\Sigma_1^2\cdots\Sigma_{D-2}^2]},
\end{equation}
and thus
\begin{equation}
c^2\frac{\Sigma_D^2}{\Tr[\Sigma_D^2]}=\frac{\Sigma_1^2}{\Tr[\Sigma]},\ \frac{\Sigma_D^2}{d_y}=c^2\frac{\Sigma_1^2}{\Tr[\Sigma_1^2]}.
\label{eq:Sigma1D_2}
\end{equation}
Combining \eqref{eq:S'}, \eqref{eq:Sigma1D_1} and \eqref{eq:Sigma1D_2}, we finish the proof.
\end{proof}

The following lemma is a restatement of Theorem 12 of \cite{ziyin2025neural}.
\begin{lemma}
\label{lemma:deep_linear} Consider a $D-$layer linear network which has more than $d$ hidden units in every layer. Then at the global minimum of \eqref{eq: free energy}, we have
\begin{equation}
\sqrt{\Sigma_\epsilon}W_{D}=L\Sigma_DU_{D-1}^T,\ W_i=U_i\Sigma_iU_{i-1}^T,\ W_1\sqrt{\Sigma_x}=U_1\Sigma_1R,
\end{equation}
for $i=2,\cdots,D-1$, where $\{U_i\}_{i=1}^{D-1}$ are arbitrary matrices satisfying $U_i^TU_i=I_{d\times d}$, and $\Sigma_x=\mathbb{E}[xx^T]$, $\Sigma_\epsilon=\mathbb{E}[\epsilon\epsilon^T]$. Moreover,
\begin{equation}
\begin{aligned}
&\Sigma_1=(\Tr S')^{-\frac{D-2}{2D}}\frac{\Tr[\Sigma_x]^{\frac{D-1}{2D}}}{\Tr[\Sigma_\epsilon]^{\frac{1}{2D}}}\sqrt{S'},\\
&\Sigma_D=(\Tr S')^{-\frac{D-2}{2D}}\frac{\Tr[\Sigma_\epsilon]^{\frac{D-1}{2D}}}{\Tr[\Sigma_x]^{\frac{1}{2D}}}\sqrt{S'},\\
&\Sigma_i=(\Tr S')^{1/D}(\Tr[\Sigma_\epsilon]\Tr[\Sigma_x])^{-\frac{1}{2D}}I_d.
\end{aligned}
\label{eq:solution}
\end{equation}
\end{lemma}

\subsection{Initial sharpness}
To understand the optimization landscape at the start of training, we analyze the expected sharpness under standard random initialization schemes.
\begin{proposition}
Consider a $D$-layer linear network where each weight matrix $W_k \in \mathbb{R}^{d_{k+1} \times d_k}$ is initialized with i.i.d. entries $w_{i,j}^{(k)} \sim \mathcal{N}(0, \sigma_k^2)$, where $d_1 = d_x$ and $d_{D+1} = d_y$. At initialization, the expected sharpness is given by:
\begin{equation}
\mathbb{E}[T(\theta)]=d_y\Pi_{k=2}^Dd_k\Tr[\Sigma_x]\sum_{k=1}^D\Pi_{j\neq k}\sigma_j^2.
\label{eq:initial_sharpness}
\end{equation}
If we further assume $\sigma_k^2 = \sigma^2$ and $d_k = d$ for all $k$, this simplifies to:
\begin{equation}
\mathbb{E}[T(\theta)] = D \cdot d_y \Tr[\Sigma_x] (d \sigma^2)^{D-1}.
\end{equation}
\label{prop:initial_sharpness}
\end{proposition}
\begin{proof}
Recall that $T(\theta) = \sum_{k=1}^D \Tr[A_k^T A_k] \Tr[B_k \Sigma_x B_k^T]$. Since the weight matrices $\{W_k\}$ are initialized independently, the expectation of the trace product decomposes into the product of expectations. For $A_k$, we have
\begin{equation}
\mathbb{E}[\Tr[A_k^TA_k]]=\mathbb{E}[\Tr[W_{k+1}^T\cdots W_DW_D^T\cdots W_{k+1}]]=d_y\Pi_{j=k+1}^Dd_j\sigma_j^2.
\end{equation}
Similarly, for $B_k$:
\begin{equation}
\mathbb{E}[\Tr[B_k\Sigma_xB_k^T]]=\mathbb{E}[\Tr[W_{k-1}\cdots W_1\Sigma_xW_1^T\cdots W_{k-1}]]=\Tr[\Sigma_x]\Pi_{j=1}^{k-1}d_{j+1}\sigma_j^2.
\end{equation}
Substituting these into the total sharpness expression \eqref{eq:T-A-B} yields \eqref{eq:initial_sharpness}.
\end{proof}
For the Xavier initialization, we have $\sigma_k^2=\frac{2}{d_k+d_{k+1}}$. Taking $d_k=d$ for all $k$, we have $\mathbb{E}[T(\theta)]=Dd_y\Tr[\Sigma_x]$. For the Kaiming initialization with linear activation, we have $\sigma^2=\frac{1}{d_k}$. Then the sharpness is the same if $d_k=d$ for any $k$.

\subsection{Largest eigenvalue}
The following theorem characterizes the spectral radius of the Hessian at the optimal solutions.
\begin{theorem}
Suppose that $\Sigma_x=\hat{R}^{-1}\Lambda_x\hat{R}$ for some diagonal matrix $\Lambda_x$, where $\hat{R}$ is the right singular matrix of $V\sqrt{\Sigma_x}$. Then, the largest eigenvalue of the Hessian matrix for the solution \eqref{eq:min_sharpness_solution} is
\begin{equation}
\begin{aligned}
\lambda_{\max}(\mathbb{E}[H(\theta)])&=(D-2)(\Tr \hat{S})^{-\frac2D}(d_y\Tr\Sigma_x)^{\frac{1}{D}}\lambda_{\max}(\hat{S})^2+(\Tr \hat{S})^{\frac{D-2}{D}}d_y^{-\frac{D-1}{D}}(\Tr\Sigma_x)^{\frac1D}\lambda_{\max}(\hat{S})\\&\qquad+(\Tr \hat{S})^{\frac{D-2}{D}}d_y^{\frac1D}(\Tr\Sigma_x)^{-\frac{D-1}{D}}\lambda_{\max}(\hat{S}\Lambda_x^{-1})\lambda_{\max}(\Sigma_x)
\end{aligned}
\label{eq:max_eig1}
\end{equation}
Moreover, suppose that $\Sigma_x=R^{-1}\Lambda_xR$ and $\Sigma_\epsilon=L\Lambda_\epsilon L^{-1}$ for some diagonal matrices $\Lambda_x,\Lambda_\epsilon$, where $L,R$ are the left and right singular matrices of $\sqrt{\Sigma_\epsilon}V\sqrt{\Sigma_x}$. Then the largest eigenvalue of the Hessian matrix for the solution \eqref{eq:solution} is
\begin{equation}
\begin{aligned}
\lambda_{\max}(\mathbb{E}[H(\theta)])&=(D-2)(\Tr S)^{-\frac2D}(\Tr\Sigma_\epsilon\Tr\Sigma_x)^{\frac{1}{D}}\lambda_{\max}(\Lambda_\epsilon^{-1}S)\lambda_{\max}(S)+(\Tr S)^{\frac{D-2}{D}}(\Tr\Sigma_\epsilon)^{-\frac{D-1}{D}}(\Tr\Sigma_x)^{\frac1D}\lambda_{\max}(S)\\&\qquad+(\Tr S)^{\frac{D-2}{D}}(\Tr\Sigma_\epsilon)^{\frac1D}(\Tr\Sigma_x)^{-\frac{D-1}{D}}\lambda_{\max}(\Lambda_\epsilon^{-1}S\Lambda_x^{-1})\lambda_{\max}(\Sigma_x)
\end{aligned}
\label{eq:max_eig2}
\end{equation}
\label{theo:max_eig}
\end{theorem}
\begin{proof}
Step 1: Quadratic Form of the Hessian.

By \eqref{eq:Hessian}, for a deep linear network $f(\theta, x) = W_D \cdots W_1 x$, the Hessian is
\begin{equation}
H(\theta)=\mathbb{E}_x\nabla f(\theta,x)\nabla f(\theta,x)^T.
\end{equation}
Let $\Delta=\text{vec}(\Delta_D,\cdots,\Delta_1)$ be the perturbation, where $\Delta_i$ has the same dimension as $W_i$, and the quadratic form of the Hessian is
\begin{equation}
\Delta^TH(\theta)\Delta=\mathbb{E}_x \left[ \left\| \sum_{k=1}^D A_k \Delta_k B_k x \right\|_2^2 \right]=\left\| \sum_{k=1}^D A_k \Delta_k B_k \sqrt{\Sigma_x} \right\|_F^2,
\label{eq:Delta-H-Delta}
\end{equation}
where $A_k:=W_D\cdots W_{k+1}$ and $B_k:=W_{k-1}\cdots W_1$, and specially $A_D=I_{d_y}$, $B_1=I$. 

Step 2: Upper Bound.

Using the triangle inequality and the sub-multiplicative property of the spectral norm, we bound the magnitude of the summed perturbations:
\begin{equation}
\sqrt{\Delta^TH(\theta)\Delta}\leq\sum_{k=1}^D||A_k||_2||B_k\sqrt{\Sigma_x}||_2||\Delta_k||_F.
\end{equation}
Applying the Cauchy-Schwarz inequality to the sum, we obtain:
\begin{equation}
\frac{\Delta^TH(\theta)\Delta}{\Delta^T\Delta}\leq\frac{\left(\sum_{k=1}^D||A_k||_2||B_k\sqrt{\Sigma_x}||_2||\Delta_k||_F\right)^2}{\sum_{k=1}^D||\Delta_k||^2_F}\leq\sum_{k=1}^D||A_k||^2_2||B_k\sqrt{\Sigma_x}||^2_2.
\end{equation}
This gives
\begin{equation}
\begin{aligned}
\lambda_{\max}(\mathbb{E}[H(\theta)]) &\le \sum_{k=1}^{D} \| W_D \dots W_{k+1} \|^2_2 \cdot \| W_{k-1} \dots W_1\sqrt{\Sigma_z} \|_2^2\\
&=\sum_{k=1}^{D}\|\hat{L}\Sigma_D\cdots\Sigma_{k+1}U_{k+1}\|_2^2\|U_{k-1}\Sigma_{k-1}\cdots\Sigma_{1}\hat{R}\|_2^2\\
&=(D-2)(\Tr \hat{S})^{-2/D}(d_y\Tr\Sigma_x)^{\frac{1}{D}}\lambda_{\max}(\hat{S})^2+(\Tr \hat{S})^{\frac{D-2}{D}}d_y^{-\frac{D-1}{D}}(\Tr\Sigma_x)^{\frac1D}\lambda_{\max}(\hat{S})\\&\qquad+(\Tr \hat{S})^{\frac{D-2}{D}}d_y^{\frac1D}(\Tr\Sigma_x)^{-\frac{D-1}{D}}\lambda_{\max}(\hat{S}\Lambda_x^{-1})\lambda_{\max}(\Sigma_x).
\end{aligned}
\end{equation}

Step 3: Tightness and Construction.

To show that the upper bound is achievable, we construct a specific perturbation $\Delta$. Let $u_D$ and $u_0$ be the top left and right singular vectors of the full network product, respectively. For the solution \eqref{eq:min_sharpness_solution}, we choose rank-1 perturbations: $\Delta_k=\alpha_ku_iu_{i-1}^T$, where $u_k$ is the column of $U_k$ corresponding to the largest singular vector of $\hat{S}$ (or $\hat{S}\Lambda_x^{-1/2}$ for $k=0$), and specially $U_0:=\hat{R},U_D:=\hat{L}$. Thus, we have
\begin{equation}
A_k\Delta_kB_k\sqrt{\Sigma_x}=\alpha_k||A_k||_2||B_k\sqrt{\Sigma_x}||_2u_Du_0^T,
\end{equation}
where we use the fact that 
\begin{equation}
W_1=U_1\Sigma_1\hat{R}\Sigma_x^{-1/2}=U_1\Sigma_1\Lambda_x^{-1/2}\hat{R}
\end{equation}
under the assumption $\Sigma_x=\hat{R}^{-1}\Lambda_x\hat{R}$. Then, we have
\begin{equation}
\Delta^TH(\theta)\Delta=\left\|\sum_{k=1}^D\alpha_k||A_k||_2||B_k\sqrt{\Sigma_x}||_2u_Du_0^T\right\|_F^2=\left(\sum_{k=1}^D\alpha_k||A_k||_2||B_k\sqrt{\Sigma_x}||_2\right)^2,
\end{equation}
which, by choosing $\alpha_k=||A_k||_2||B_k\sqrt{\Sigma_x}||_2$, gives
\begin{equation}
\frac{\Delta^TH(\theta)\Delta}{\Delta^T\Delta}=\sum_{k=1}^D||A_k||^2_2||B_k\sqrt{\Sigma_x}||^2_2.
\end{equation} This proves \eqref{eq:max_eig1}.

For the solution \eqref{eq:solution}, we have
\begin{equation}
W_D=\Sigma_{\epsilon}^{-1/2}L\Sigma_DU_{D-1}^T=L\Lambda_\epsilon^{-1/2}\Sigma_DU_{D-1}^T
\end{equation}
under the assumption that $\Sigma_\epsilon=L\Lambda_\epsilon L^{-1}$. Then we can prove \eqref{eq:max_eig2} in the same manner.
\end{proof}
\begin{remark}
From \eqref{eq:Delta-H-Delta}, one obtains
\begin{equation}
\Delta^TH(\theta)\Delta=\left\|\left[\sum_{k=1}^D(B_k\sqrt{\Sigma_x})^T\otimes A_k\right]\text{vec}(\Delta_k)\right\|_2^2,
\end{equation}
which implies
\begin{equation}
\lambda_{\max}(H(\theta))=\sigma_{\max}\left(\sum_{k=1}^D(B_k\sqrt{\Sigma_x})^TB_k\sqrt{\Sigma_x}\otimes A_kA_k^T\right),
\label{eq:lambda_max}
\end{equation}
which recovers Theorem 3.1 of a concurrent work \citep{kamber2025sharpness}. However, unless $B_k\sqrt{\Sigma_x}$ and $\Sigma_x$ share the same singular spaces, \eqref{eq:lambda_max} cannot be further simplified under the solution \eqref{eq:min_sharpness_solution}. Thus, we introduce the commutation assumption in Theorem \ref{theo:max_eig}.
\end{remark}

\subsection{Proof of Corollary \ref{cor1}}
\begin{proof}
We begin by invoking Von Neumann's trace inequality to derive three useful bounds:
\begin{equation}
\lambda_{\min}(\sqrt{\Sigma_\epsilon})||V\sqrt{\Sigma_x}||_*\leq\Tr S\leq \lambda_{\max}(\sqrt{\Sigma_\epsilon})||V\sqrt{\Sigma_x}||_*,
\label{ineq:trace1}
\end{equation}
\begin{equation}
\Tr[\Sigma_\epsilon^{-1}LSL^T]\geq\frac{\sigma_{\min}(S)}{\lambda_{\min}(\Sigma_\epsilon)}\geq\frac{\sigma_{\min}(V\sqrt{\Sigma_x})}{\sqrt{\lambda_{\min}(\Sigma_\epsilon)}},
\label{ineq:trace2}
\end{equation}
and
\begin{equation}
\Tr[\Sigma_\epsilon^{-1}LSL^T]\leq\frac{||S||_*}{\lambda_{\min}(\Sigma_\epsilon)}\leq\frac{\lambda_{\max}(\sqrt{\Sigma_\epsilon})||V\sqrt{\Sigma_x}||_*}{\lambda_{\min}(\Sigma_\epsilon)},
\label{ineq:trace3}
\end{equation}
where $||\cdot||_*$ denotes the kernel norm. 

To find the lower bound for $T(\theta)$, we utilize the first term of the sharpness expression. Applying \eqref{ineq:trace1} and \eqref{ineq:trace2}, we obtain:
\begin{equation}
\begin{aligned}
\left(\Tr S\right)^{(D-2)/D}(\Tr\Sigma_\epsilon)^{1/D}\Tr[\Sigma_\epsilon^{-1}LSL^T]&\geq\left(\lambda_{\min}(\sqrt{\Sigma_\epsilon})||V\sqrt{\Sigma_x}||_*\right)^{(D-2)/D}(\Tr\Sigma_\epsilon)^{1/D}\frac{\sigma_{\min}(V\sqrt{\Sigma_x})}{\sqrt{\lambda_{\min}(\Sigma_\epsilon)}}\\&\geq \max(\sigma_{\min}(V\sqrt{\Sigma_x})^2,1)\frac{\lambda_{\max}(\Sigma_\epsilon)^{\frac{D-2}{2D}+\frac1D}}{\sqrt{\lambda_{\min}(\Sigma_\epsilon)}}\\
&:=2c_1\sqrt{\kappa(\Sigma_\epsilon)},
\end{aligned}
\end{equation}
where $c_1$ is a constant depending solely on the data and model through $V\sqrt{\Sigma_x}$. It follows from \eqref{eq:sharpness} that the total sharpness is lower-bounded by the square root of the condition number:
\begin{equation}
T(\theta)\geq 2(D-1)c_1\sqrt{\kappa(\Sigma_\epsilon)}\geq Dc_1\sqrt{\kappa(\Sigma_\epsilon)}.
\end{equation}

For the upper bound, we consider both terms of the sharpness. First, using \eqref{ineq:trace1} and \eqref{ineq:trace3}, we bound the primary term:
\begin{equation}
\begin{aligned}
&\left(\Tr S\right)^{(D-2)/D}(\Tr\Sigma_\epsilon)^{1/D}\Tr[\Sigma_\epsilon^{-1}LSL^T]\\&\leq(\lambda_{\max}(\sqrt{\Sigma_\epsilon})||V\sqrt{\Sigma_x}||_*)^{(D-2)/D}(d_y\lambda_{\max}(\sqrt{\Sigma_\epsilon}))^{1/D}\frac{\lambda_{\max}(\sqrt{\Sigma_\epsilon})||V\sqrt{\Sigma_x}||_*}{\lambda_{\min}(\Sigma_\epsilon)}\\&\leq \max(d_y^{1/D}||V\sqrt{\Sigma_x}||^2_*,1)\frac{\lambda_{\max}(\Sigma_\epsilon)}{\lambda_{\min}(\Sigma_\epsilon)}.
\end{aligned}
\end{equation}
Next, we bound the second term in the expression of $T(\theta)$ using \eqref{ineq:trace1}:
\begin{equation}
\begin{aligned}
d_y\left(\Tr S\right)^{2(D-1)/D}\frac{(\Tr\Sigma_x)^{1/D}}{(\Tr\Sigma_\epsilon)^{(D-1)/D}}&\leq d_y(\Tr\Sigma_x)^{1/D}\frac{(\lambda_{\max}(\sqrt{\Sigma_\epsilon})||V\sqrt{\Sigma_x}||_*)^{2(D-1)/D}}{(d_y\lambda_{\min}(\Sigma_\epsilon))^{(D-1)/D}}\\
&\leq d_yd_x\max(||V\sqrt{\Sigma_x}||_*^2,1)\frac{\lambda_{\max}(\Sigma_\epsilon)}{\lambda_{\min}(\Sigma_\epsilon)}.
\end{aligned}
\end{equation}
By plugging these two upper bounds into \eqref{eq:sharpness}, we arrive at the final linear relation:
\begin{equation}
T(\theta)\leq (D-1)\max(d_y||V\sqrt{\Sigma_x}||^2_*,1)\frac{\lambda_{\max}(\Sigma_\epsilon)}{\lambda_{\min}(\Sigma_\epsilon)}+d_yd_x\max(||V\sqrt{\Sigma_x}||_*^2,1)\frac{\lambda_{\max}(\Sigma_\epsilon)}{\lambda_{\min}(\Sigma_\epsilon)}\leq Dc_2\kappa(\Sigma_\epsilon),
\end{equation}
where $c_2$ is a constant determined by $V\sqrt{\Sigma_x}$ and the dimensions $d_x, d_y$. This concludes the proof.
\end{proof}

\subsection{Proof of Corollary \ref{cor2}}
\begin{proof}
The proof follows from a direct comparison of the minimal sharpness conditions derived in Lemma \ref{lemma:min_sharpness} and the SGD solutions characterized in Lemma \ref{lemma:deep_linear}. By setting $\Sigma_\epsilon = I$ in \eqref{eq:solution}, the matrix balancing and the singular value distributions across layers align perfectly with the requirements for minimal sharpness \eqref{eq:min_sharpness_solution}. 
\end{proof}
As a sanity check, when we substitute $\Sigma_\epsilon=I$ into \eqref{eq:sharpness}, we have $V'=V\sqrt{\Sigma_x}=LSL^T$, $S=\hat{S}$, and thus
\begin{equation}
T(\theta)=D(\Tr S)^{2(D-1)/D}(d_y\Tr\Sigma_x)^{1/D},
\end{equation}
which is the same as the minimal sharpness.

\subsection{Proof of Corollary \ref{cor4}}
\begin{proof}
nder the commutation assumption, the eigenvectors of the covariance matrices $\Sigma_\epsilon, \Sigma_x$ and the singular vectors of the model $V$ align. This allows for a simplified spectral analysis:
\begin{equation}
\Sigma_\epsilon=L\Lambda_\epsilon L^{-1},\ \Sigma_x=R^{-1}\Lambda_xR,\ \sqrt{\Sigma_\epsilon}V\sqrt{\Sigma_x}=LSR.
\end{equation}
Consequently, we can bound the spectral radius of the product matrix $S$ and its scaled variants:
\begin{equation}
\lambda_{\max}(\sqrt{\Sigma_\epsilon})\sigma_{\min}(V\sqrt{\Sigma_x})\leq\lambda_{\max}(S)\leq\lambda_{\max}(\sqrt{\Sigma_\epsilon})\sigma_{\max}(V\sqrt{\Sigma_x})
\end{equation}
and
\begin{equation}
\frac{\lambda_{\min}(V\sqrt{\Sigma_x})}{\lambda_{\min}(\sqrt{\Sigma_\epsilon})}\leq\lambda_{\max}(\Lambda_\epsilon^{-1}S)\leq\frac{\lambda_{\max}(V\sqrt{\Sigma_x})}{\lambda_{\min}(\sqrt{\Sigma_\epsilon})}.
\end{equation}

For the upper bound, we evaluate the three terms in \eqref{eq:max_eig2} individually using the trace inequality \eqref{ineq:trace1}. For the first term:
\begin{equation}
\begin{aligned}
&(\Tr S)^{-\frac2D}(\Tr\Sigma_\epsilon\Tr\Sigma_x)^{\frac{1}{D}}\lambda_{\max}(\Lambda_\epsilon^{-1}S)\lambda_{\max}(S)\\&\leq(\lambda_{\min}(\sqrt{\Sigma_\epsilon})||V\sqrt{\Sigma_x}||_*)^{-\frac2D}(d_y\lambda_{\max}(\Sigma_\epsilon)\Tr\Sigma_x)^{\frac1D}\frac{\sigma_{\min}(V\sqrt{\Sigma_x})}{\lambda_{\min}(\sqrt{\Sigma_\epsilon})}\lambda_{\max}(\sqrt{\Sigma_\epsilon})\sigma_{\min}(V\sqrt{\Sigma_x})\\
&\leq\max(\sigma_{\min}(V\sqrt{\Sigma_x})^{2},1)\max(d_y\Tr\Sigma_x,1)\kappa(\Sigma_\epsilon).
\end{aligned}
\end{equation}
For the second term:
\begin{equation}
\begin{aligned}
&(\Tr S)^{\frac{D-2}{D}}(\Tr\Sigma_\epsilon)^{-\frac{D-1}{D}}(\Tr\Sigma_x)^{\frac1D}\lambda_{\max}(S)\\&\leq(\lambda_{\max}(\sqrt{\Sigma_\epsilon})||V\sqrt{\Sigma_x}||_*)^{\frac{D-2}{D}}\lambda_{\min}(\Sigma_\epsilon)^{-\frac{D-1}{D}}(\Tr\Sigma_x)^{\frac1D}\lambda_{\max}(\sqrt{\Sigma_\epsilon})\sigma_{\max}(V\sqrt{\Sigma_x})\\
&\leq\max(||V\sqrt{\Sigma_x}||_*^2,1)\max(\Tr\Sigma_x,1)\kappa(\Sigma_\epsilon).
\end{aligned}
\end{equation}
And for the third term:
\begin{equation}
\begin{aligned}
&(\Tr S)^{\frac{D-2}{D}}(\Tr\Sigma_\epsilon)^{\frac1D}(\Tr\Sigma_x)^{-\frac{D-1}{D}}\lambda_{\max}(\Lambda_\epsilon^{-1}S\Lambda_x^{-1})\lambda_{\max}(\Sigma_x)\\
&\leq (\lambda_{\max}(\sqrt{\Sigma_\epsilon})||V\sqrt{\Sigma_x}||_*)^{\frac{D-2}{D}}\lambda_{\max}(\Sigma_\epsilon)^{\frac1D}(\Tr\Sigma_x)^{-\frac{D-1}{D}}\frac{\sigma_{\max}(V\sqrt{\Sigma_x})}{\lambda_{\min}(\sqrt{\Sigma_\epsilon})\lambda_{\min}(\Sigma_x)}\lambda_{\max}(\Sigma_x)\\
&\leq\max(||V\Sigma_x||_*^2,1)\max(\lambda_{\max}(\Sigma_x)^{-1},1)\kappa(\Sigma_x)\kappa(\Sigma_\epsilon).
\end{aligned}
\end{equation}
Summing these contributions, we conclude that the spectral radius of the Hessian is bounded by
\begin{equation}
\lambda_{\max}(\mathbb{E}[H(\theta)])\leq c_2D\kappa(\Sigma_\epsilon)
\end{equation}
for a constant $c_2$ determined by the problem geometry ($V, \Sigma_x, d_y$).

We derive the lower bound in the same way. For the first term:
\begin{equation}
\begin{aligned}
&(\Tr S)^{-\frac2D}(\Tr\Sigma_\epsilon\Tr\Sigma_x)^{\frac{1}{D}}\lambda_{\max}(\Lambda_\epsilon^{-1}S)\lambda_{\max}(S)\\&\geq(\lambda_{\max}(\sqrt{\Sigma_\epsilon})||V\sqrt{\Sigma_x}||_*)^{-\frac2D}(\lambda_{\max}(\Sigma_\epsilon)\Tr\Sigma_x)^{\frac1D}\frac{\sigma_{\min}(V\sqrt{\Sigma_x})}{\lambda_{\min}(\sqrt{\Sigma_\epsilon})}\lambda_{\max}(\sqrt{\Sigma_\epsilon})\sigma_{\min}(V\sqrt{\Sigma_x})\\
&\geq\min(\sigma_{\min}(V\sqrt{\Sigma_x})^{2},1)\min(\Tr\Sigma_x,1)\sqrt{\kappa(\Sigma_\epsilon)}.
\end{aligned}
\end{equation}
For the second term:
\begin{equation}
\begin{aligned}
&(\Tr S)^{\frac{D-2}{D}}(\Tr\Sigma_\epsilon)^{-\frac{D-1}{D}}(\Tr\Sigma_x)^{\frac1D}\lambda_{\max}(S)\\&\geq(\lambda_{\max}(\sqrt{\Sigma_\epsilon})\sigma_{\min}(V\sqrt{\Sigma_x}))^{\frac{D-2}{D}}\lambda_{\max}(\Sigma_\epsilon)^{-\frac{D-1}{D}}(\Tr\Sigma_x)^{\frac1D}\lambda_{\max}(\sqrt{\Sigma_\epsilon})\sigma_{\min}(V\sqrt{\Sigma_x})\\
&\geq\min(\sigma_{\min}(V\sqrt{\Sigma_x})^2,1)\min(\Tr\Sigma_x,1)\sqrt{\kappa(\Sigma_\epsilon)}.
\end{aligned}
\end{equation}
Finally, for the third term:
\begin{equation}
\begin{aligned}
&(\Tr S)^{\frac{D-2}{D}}(\Tr\Sigma_\epsilon)^{\frac1D}(\Tr\Sigma_x)^{-\frac{D-1}{D}}\lambda_{\max}(\Lambda_\epsilon^{-1}S\Lambda_x^{-1})\lambda_{\max}(\Sigma_x)\\
&\geq (\lambda_{\max}(\sqrt{\Sigma_\epsilon})\sigma_{\min}(V\sqrt{\Sigma_x}))^{\frac{D-2}{D}}\lambda_{\max}(\Sigma_\epsilon)^{\frac1D}(\Tr\Sigma_x)^{-\frac{D-1}{D}}\frac{\sigma_{\min}(V\sqrt{\Sigma_x})}{\lambda_{\min}(\sqrt{\Sigma_\epsilon})\lambda_{\min}(\Sigma_x)}\lambda_{\max}(\Sigma_x)\\
&\geq\min(\sigma_{\min}(V\Sigma_x)^2,1)\min(d_y^{-1}\lambda_{\max}(\Sigma_x)^{-1},1)\kappa(\Sigma_x)\sqrt{\kappa(\Sigma_\epsilon)}.
\end{aligned}
\end{equation}
Assembling these lower bounds in \eqref{eq:max_eig2} yields:
\begin{equation}
\lambda_{\max}(\mathbb{E}[H(\theta)])\geq c_1D\sqrt{\kappa(\Sigma_\epsilon)}
\end{equation}
for some constant $c_1$ depending on $V$, $\Sigma_x$ and $d_y$. This completes the proof.
\end{proof}

\subsection{Proof of Corollary \ref{cor5}}
\begin{proof}
It is proven by plugging $\Tr\Sigma_\epsilon,\Tr\Sigma_x,\Tr S=\Theta(d^{\max(0,1-a)}),\Theta(d^{\max(0,1-b)}),\Theta(d^{\max(0,1-a/2-b/2-c)})$ into \eqref{eq:sharpness}.
\end{proof}

\subsection{An Example: Power-law Spectrum} 
\label{app:power-law}
\begin{corollary}
Suppose that $\Sigma_\epsilon,\Sigma_x,V$ have power-law spectrum $\{i^{-a}\}_{i=1}^d,\{i^{-b}\}_{i=1}^d,\{i^{-c}\}_{i=1}^d$ for $a,b,c>0$ and they commute. Then for large $d$ we have
\begin{equation}
\begin{aligned}
T(\theta)&=\Theta\left((D-1)d^{\kappa_1}+d^{\kappa_2}\right)
\end{aligned}
\end{equation}
at the global minimum of \eqref{eq: free energy},
where
\begin{equation}
\begin{aligned}
\kappa_1:&=\frac{D-2}{D}\max\left(0,1-\frac a2-\frac b2-c\right)+\frac{1}{D}\max(0,1-a)\\&+\frac{1}{D}\max(0,1-b)+\max\left(0,1+\frac{a}{2}-\frac{b}{2}-c\right)
\end{aligned}
\end{equation}
and
\begin{equation}
\begin{aligned}
\kappa_2:&=1+\frac{2(D-1)}{D}\max(0,1-\frac a2-\frac b2-c)\\&+\frac{1}{D}\max(0,1-b)-\frac{D-1}{D}\max(0,1-b).
\end{aligned}
\end{equation}
\label{cor5}
\end{corollary}

Corollary \ref{cor5} further verifies that the noise imbalance is the main cause of sharpening, because the sharpness only increases when $a$ increases, but not when $b,c$ increase.

\clearpage
\section{Experiments}
\label{app:exp}

\begin{figure}[t!]
    \centering
    \includegraphics[width=0.3\linewidth]{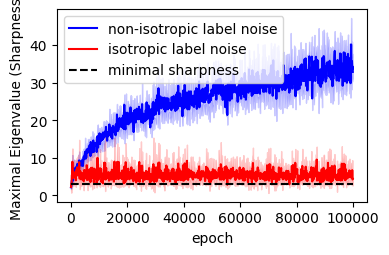}
    \includegraphics[width=0.3\linewidth]{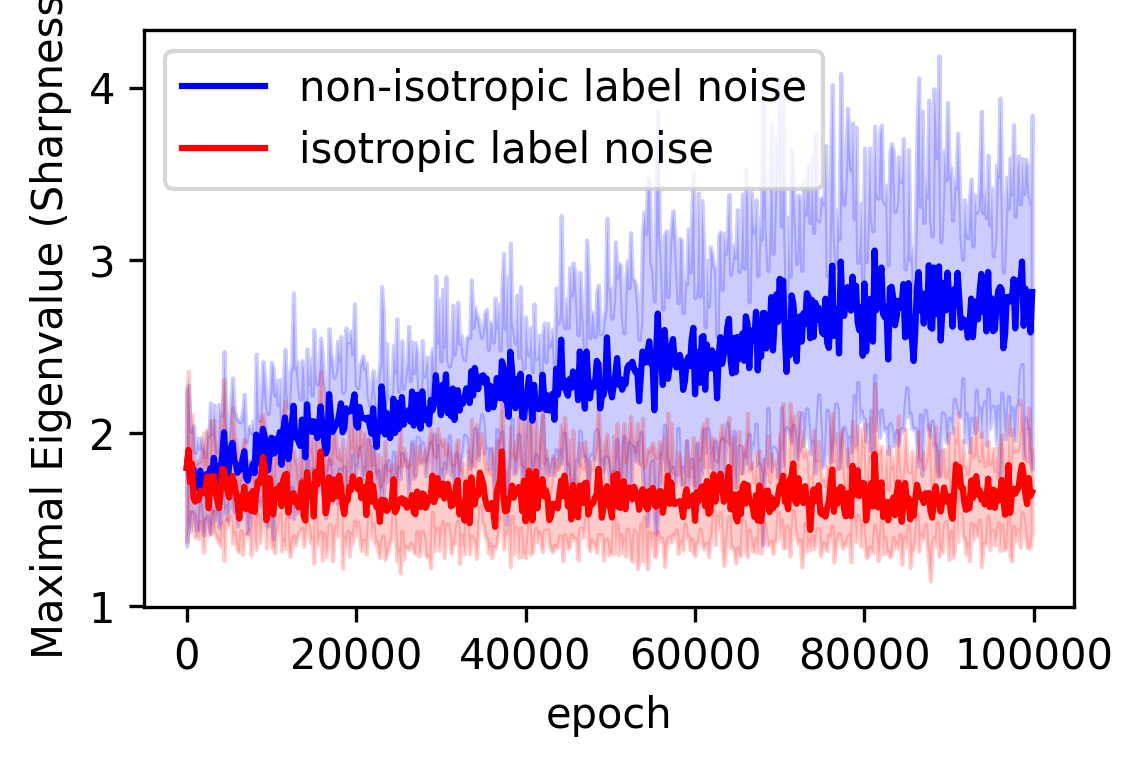}
    \includegraphics[width=0.3\linewidth]{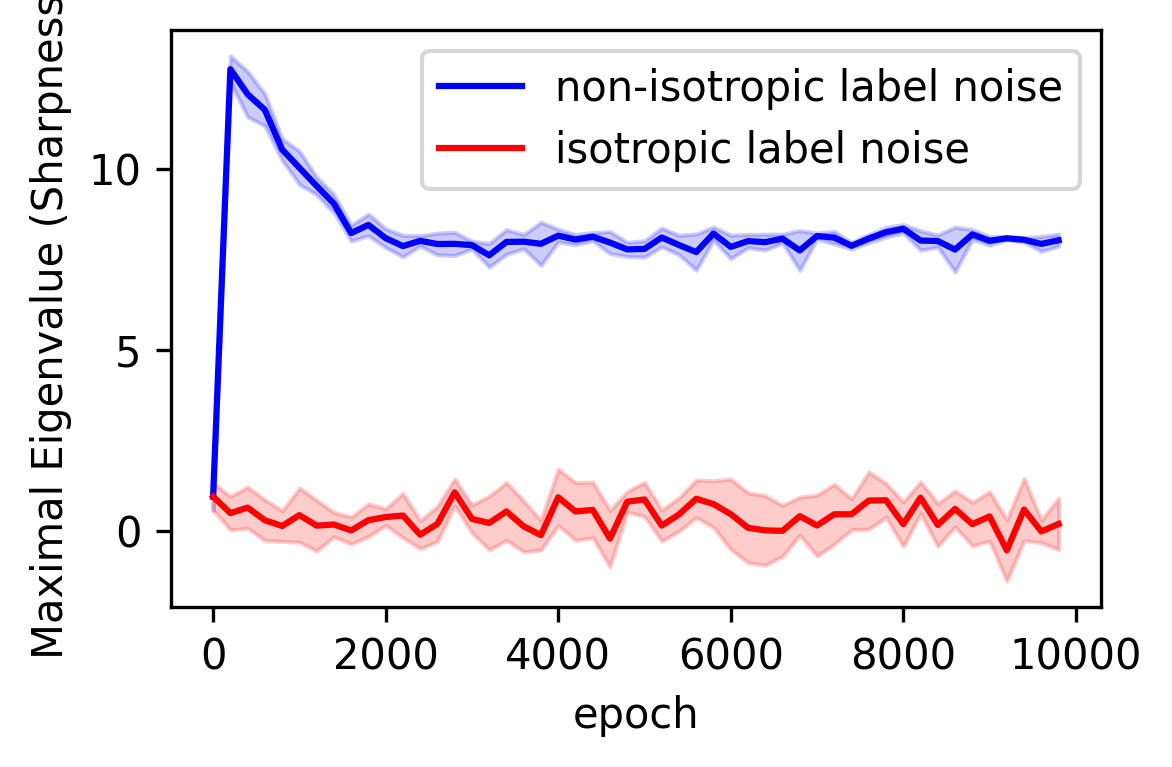}
    \caption{\small The maximal eigenvalue of the Hessian evolves similarly to the trace. \textbf{Left}: linear networks, with the minimal sharpness predicted by \eqref{eq:max_eig1}. \textbf{Middle}: Two layer ReLU networks trained under a teacher-student setting. \textbf{Right}: Two layer ReLU networks trained on MNIST. Each line is averaged over five trials shown with the standard error.}
    \label{fig:label_noise_eig}
\end{figure}

\begin{figure}[t!]
    \centering
    \includegraphics[width=0.3\linewidth]{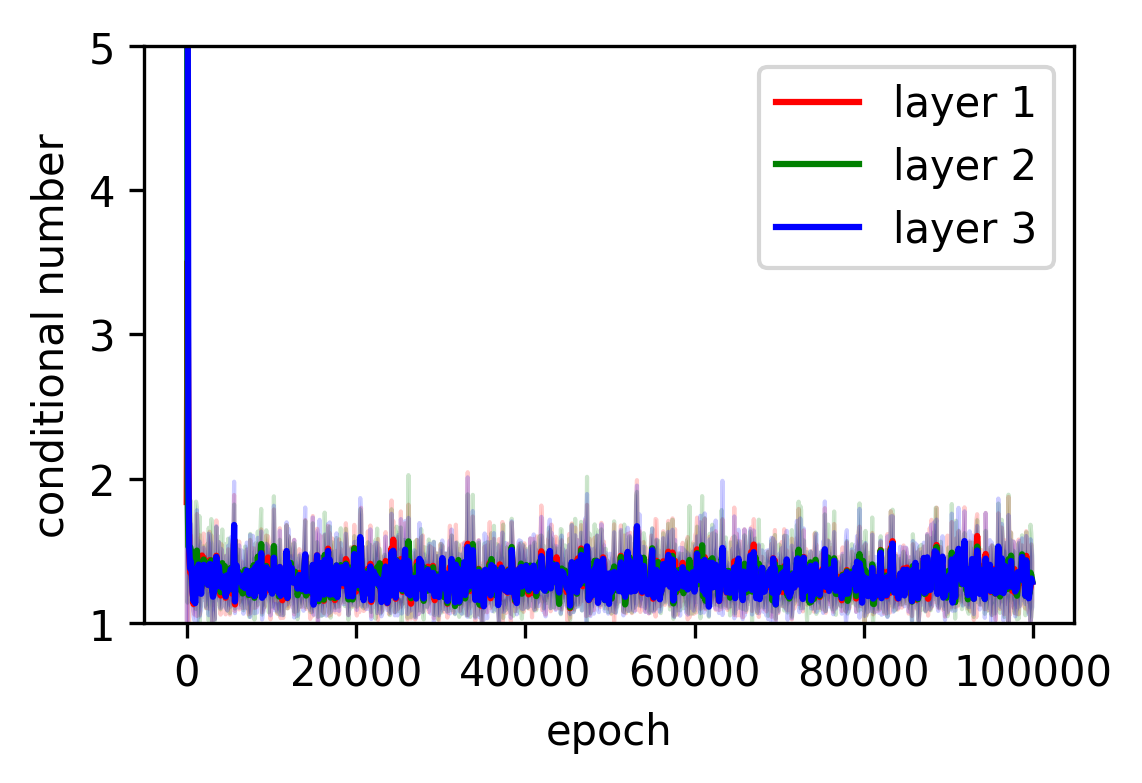}
    \includegraphics[width=0.3\linewidth]{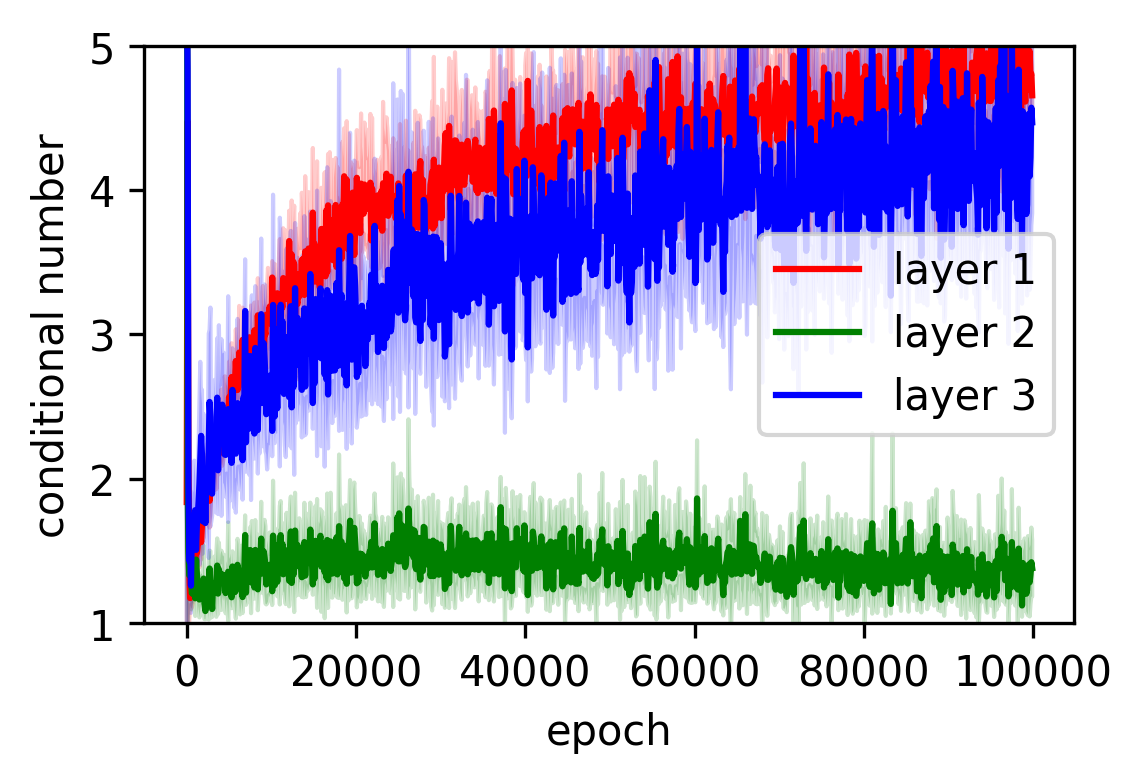}
    \caption{\small The evolution of the condition numbers for three-layer linear networks. \textbf{Left}: Balanced noise in the labels. \textbf{Right}: Imbalanced noise in the labels. }
    \label{fig:condition_number}
\end{figure}

\begin{figure}[t!]
    \centering
    \includegraphics[width=0.3\linewidth]{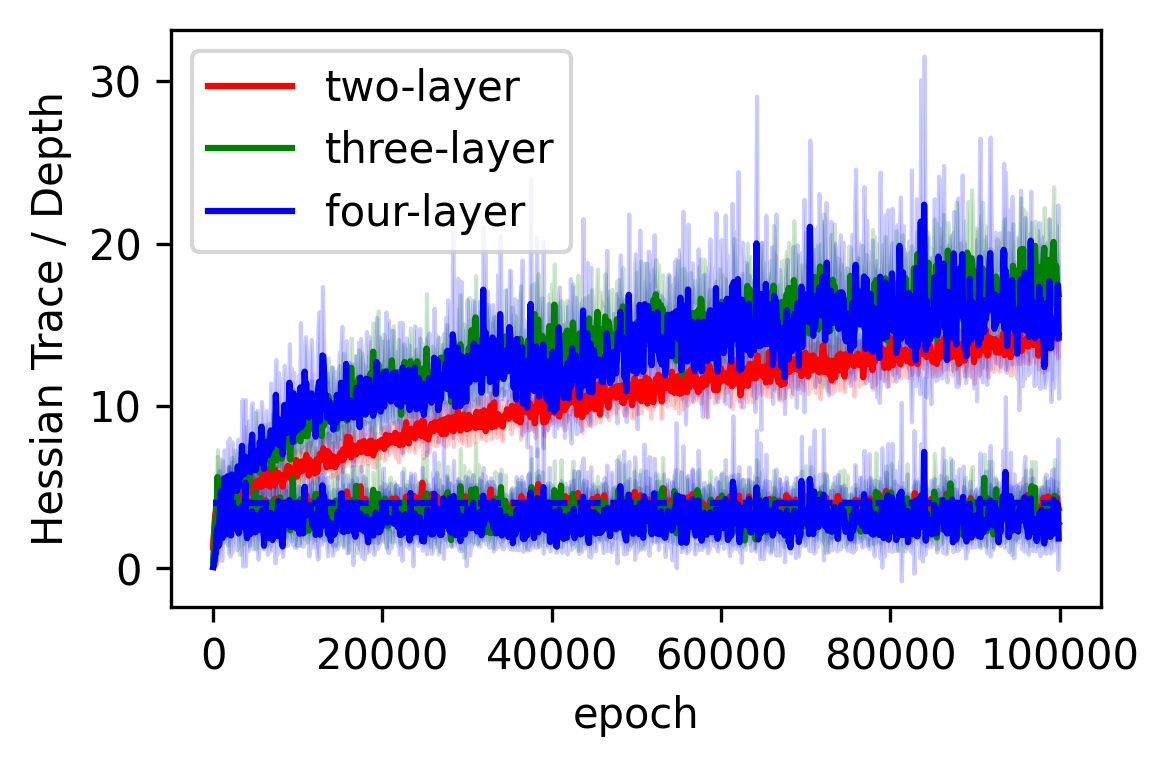}
    \includegraphics[width=0.3\linewidth]{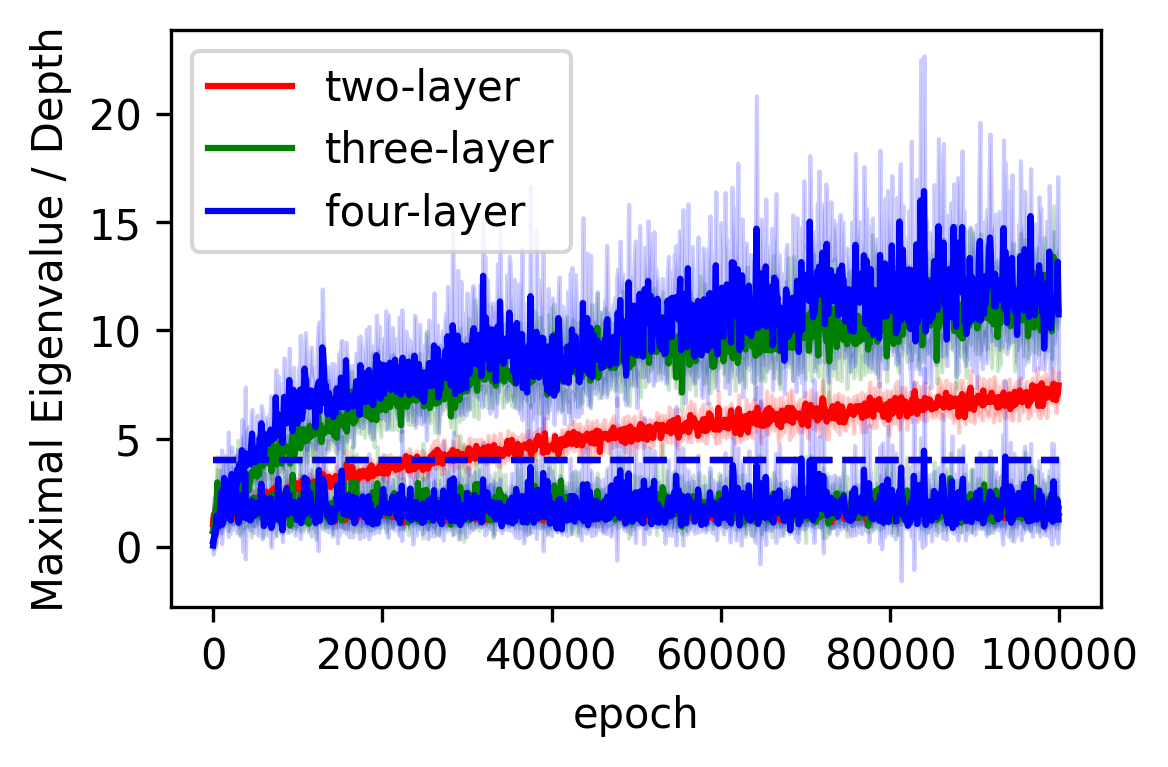}
    \caption{\small The evolution of sharpness for linear networks of various depths. \textbf{Left}: Traces of Hessian. \textbf{Right}: Maximal eigenvalues of Hessian.}
    \label{fig:linear_network_additional}
\end{figure}

\paragraph{Experiments on linear networks}

For the linear network settings, we set $D=3$, $d_x=d_y=2$, and $V=I$, with data drawn from a standard Gaussian distribution ($\Sigma_x=I$). We configure the noise scales as $\sqrt{\Sigma_\epsilon}=\text{diag}(10, 0.1)$ for the non-isotropic case and $\sqrt{\Sigma_\epsilon}=\text{diag}(10, 10)$ for the isotropic case. Training is performed using online SGD for $10^5$ epochs (learning rate $0.01$, batch size $10$). The Hessian is estimated via $200$ random samples, where the trace is calculated exactly and the maximal eigenvalue is computed using the power method. For experiments in Figure \ref{fig:initialization} and Figure \ref{fig:GD}, we maintain these hyperparameters but adjust the initialization strength and learning rate (to $0.001$), or transition to a single-batch gradient descent (GD) setting (batch size $1000$).

To complement the observations in Figure \ref{fig:label_noise}, we evaluate the maximal eigenvalue of the Hessian matrix, as shown in Figure \ref{fig:label_noise_eig}. The resulting curves exhibit consistent trends, further corroborating Corollary \ref{cor4}.

Figure \ref{fig:condition_number} illustrates the condition numbers across layers for a three-layer linear network. Aligning with our analytical solution \eqref{eq:solution}, we observe that balanced noise in the labels leads to small condition numbers across all layers. Conversely, imbalanced noise primarily inflates the condition numbers of the first and last layers, while intermediate layers remain stable. Furthermore, Figure \ref{fig:linear_network_additional} demonstrates that depth acts only as a multiplicative factor for sharpness, as suggested by Corollary \ref{cor1}.

\clearpage

\begin{figure}[t!]
    \centering
    \includegraphics[width=0.31\linewidth]{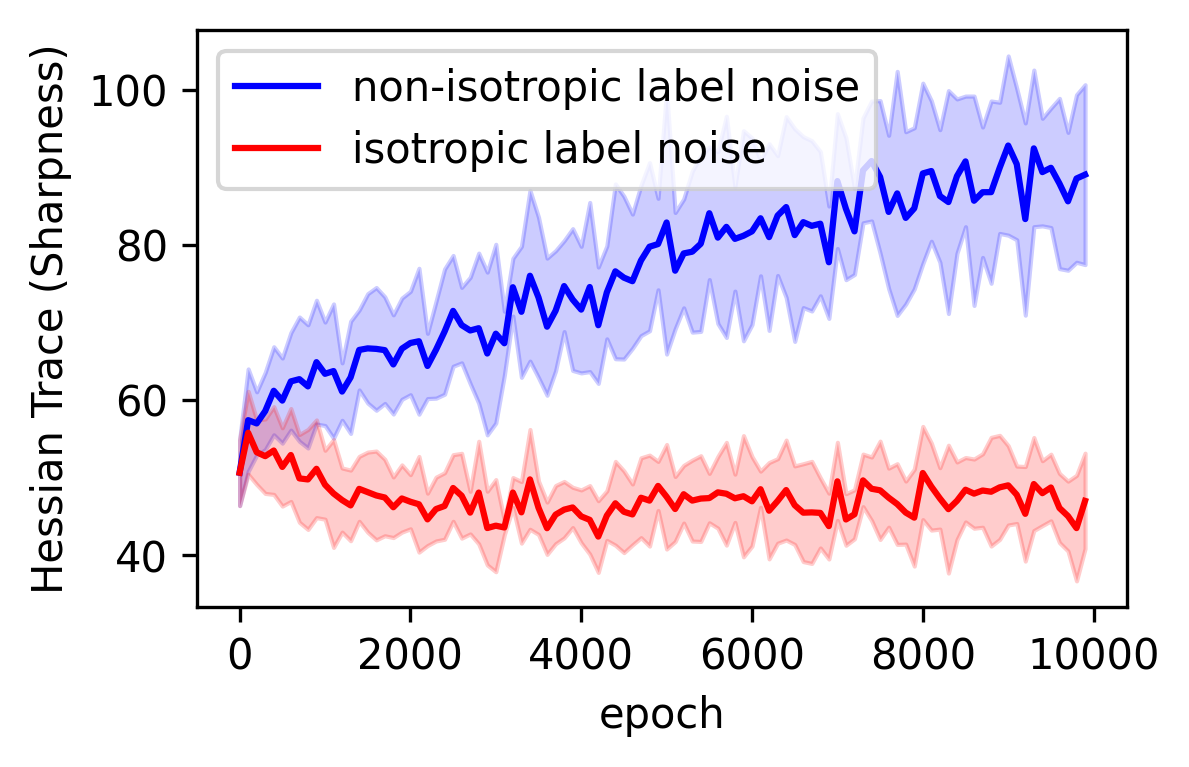}
    \includegraphics[width=0.3\linewidth]{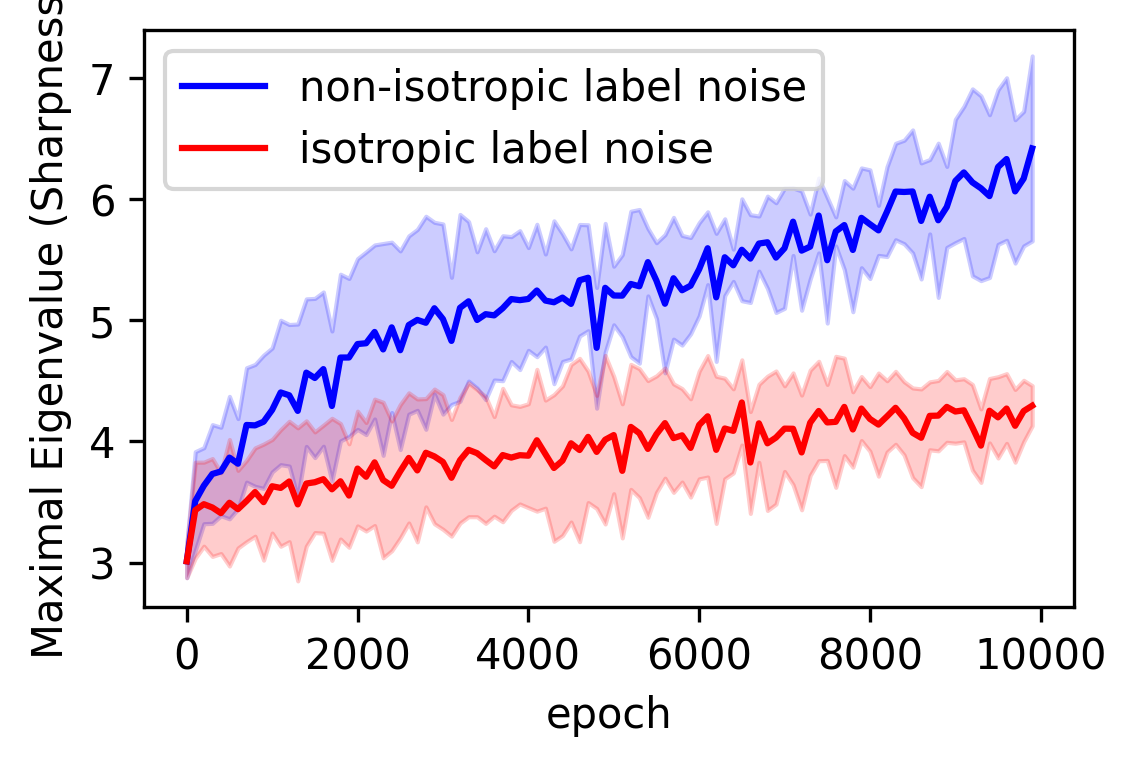}
    \caption{\small Non-isotropic noise leads to progressive sharpening also for transformers. \textbf{Left}: The evolution of the Hessian trace. \textbf{Right}: The evolution of the maximal eigenvalue.}
    \label{fig:transformer}
\end{figure}
\paragraph{Experiments on non-linear settings}
In the ReLU teacher-student setting (in Figures \ref{fig:label_noise} and \ref{fig:label_noise_eig}), we employ a single-hidden-layer teacher-student framework with $d_x=d_y=2$. The teacher and student networks consist of $100$ and $4$ hidden units, respectively, both initialized using the Kaiming method. We use MSE loss and online SGD (learning rate $0.1$, batch size $100$) for $10^5$ epochs. Noise scales are set to $\text{diag}(3, 0.03)$ and $\text{diag}(3, 3)$. Hessian statistics are computed using $200$ samples, with the trace approximated by the Hutchinson estimator.

For the MNIST dataset (in Figures \ref{fig:label_noise} and \ref{fig:label_noise_eig}), a two-layer ReLU network ($128$ hidden units) is trained using cross-entropy loss. noise in the labels is introduced by adding Gaussian noise to one-hot labels, followed by a softmax operation. The non-isotropic noise covariance $\sqrt{\Sigma_\epsilon}$ is set to $\text{diag}(10, \dots, 10, 1, \dots, 1)$ (five entries each), while the isotropic case uses $\sqrt{\Sigma_\epsilon}=10I$. We run SGD for $10^4$ epochs (learning rate $0.01$, batch size $128$) and evaluate the Hessian on a test batch of size $1024$.

To verify the universality of our findings beyond linear networks and MLPs, we extend our experiments to transformers and RNNs in a teacher-student framework. In Figures \ref{fig:sharpness_condition_number} and \ref{fig:transformer}, we employ a simplified Transformer comprising a two-head self-attention layer followed by an MLP with residual connections. The model takes inputs of dimension $8$ and sequence length $4$, with the output dimension matching the input. The hidden dimension of the MLP is set to $100$ for the teacher and $4$ for the student, both initialized via the Kaiming method.

Training is conducted using online SGD (learning rate $0.03$, batch size $300$) for $10^4$ epochs on standard Gaussian data. The Hessian is evaluated on a test set of size $1000$. We configure the noise scales as $\sqrt{\Sigma_\epsilon}=\text{diag}(4, 0.04, \dots, 0.04)$ for the non-isotropic case and $\sqrt{\Sigma_\epsilon}=\text{diag}(4, \dots, 4)$ for the isotropic case. The results in Figure \ref{fig:transformer} confirm that non-isotropic noise in the labels induces the same progressive sharpening phenomenon in transformers as observed in simpler models.

Furthermore, we investigate the specific impact of the noise covariance condition number on sharpness in Figure \ref{fig:sharpness_condition_number}. Maintaining the same experimental setup, we systematically vary the smaller eigenvalues of $\Sigma_\epsilon$ while keeping other hyperparameters fixed. This analysis is also extended to RNNs, where the teacher and student hidden units are set to $100$ and $4$, respectively. For experiments involving cross-entropy loss, an additional softmax layer is applied to the network output. Across all these diverse settings—including different architectures and loss functions—the observed trends remain highly consistent, further validating the robustness of our theoretical predictions.

\end{document}